\documentclass[11pt]{article}

\usepackage[margin=1in]{geometry}
\usepackage{times}
\usepackage{natbib}
\usepackage{hyperref}
\usepackage{url}
\usepackage{amsmath,amssymb,amsfonts}
\usepackage{graphicx}
\usepackage{booktabs}

\usepackage[T1]{fontenc}
\usepackage[utf8]{inputenc}

\usepackage{amsmath, amssymb, amsfonts, amsthm}
\usepackage{mathtools}
\usepackage{mathrsfs}
\usepackage{stmaryrd}
\usepackage{empheq}

\usepackage{graphicx}
\usepackage{wrapfig}
\usepackage{xcolor}

\usepackage{booktabs}
\usepackage{nicefrac}
\usepackage{enumitem}

\usepackage{algorithm}
\usepackage{algorithmic}

\usepackage{caption}
\usepackage{subcaption}

\usepackage{microtype}
\usepackage{hyperref}

\usepackage{etoolbox}
\usepackage{appendix}
\usepackage{wrapfig}

\usepackage{aliascnt}
\usepackage{cleveref}
\AtBeginEnvironment{appendices}{\crefalias{section}{appendix}}
\usepackage{tikz}
\mathtoolsset{showonlyrefs}
\usetikzlibrary{calc}

\newtheorem{rem}{Remark}
\newaliascnt{defn}{rem}
\newaliascnt{prop}{rem}
\newaliascnt{conjecture}{rem}
\newaliascnt{example}{rem}
\newaliascnt{ex}{rem}
\newaliascnt{lem}{rem}
\newaliascnt{cor}{rem}
\newaliascnt{thm}{rem}
\newaliascnt{ass}{rem}

\newtheorem{prop}[prop]{Proposition}
\newtheorem{lem}[lem]{Lemma}
\newtheorem{cor}[cor]{Corollary}
\newtheorem{thm}[thm]{Theorem}

\aliascntresetthe{defn}
\aliascntresetthe{prop}
\aliascntresetthe{conjecture}
\aliascntresetthe{example}
\aliascntresetthe{ex}
\aliascntresetthe{lem}
\aliascntresetthe{cor}
\aliascntresetthe{thm}
\aliascntresetthe{ass}
\crefname{prop}{proposition}{propositions}
\Crefname{prop}{Proposition}{Propositions}
\crefname{lem}{lemma}{lemmas}
\Crefname{lem}{Lemma}{Lemmas}
\crefname{cor}{corollary}{corollaries}
\Crefname{cor}{Corollary}{Corollaries}
\crefname{thm}{theorem}{theorems}
\Crefname{thm}{Theorem}{Theorems}
\crefname{defn}{definition}{definitions}
\Crefname{defn}{Definition}{Definitions}
\crefname{rem}{remark}{remarks}
\Crefname{rem}{Remark}{Remarks}
\crefname{ass}{assumption}{assumptions}
\Crefname{ass}{Assumption}{Assumptions}
\crefname{ex}{example}{examples}
\Crefname{ex}{Example}{Examples}
\crefname{example}{example}{examples}
\Crefname{example}{Example}{Examples}
\crefname{conjecture}{conjecture}{conjectures}
\Crefname{conjecture}{Conjecture}{Conjectures}
\newcommand{\norm}[1]{\left\|{#1}\right\|}

\newcommand{\eqdef}{\coloneqq}
\newcommand{\id}{\operatorname{Id}}

\newcommand{\R}{\mathbb{R}}

\newcommand{\softmax}{\operatorname{softmax}}
\newcommand{\Span}{\operatorname{span}}

\newcommand{\E}[1]{\mathbb{E}\left[{#1}\right]}

\definecolor{linkcolor}{RGB}{82, 82, 192}
\hypersetup{
    colorlinks=true,
    linkcolor=linkcolor,
    citecolor=linkcolor
}

\definecolor{codepurple}{rgb}{0.58,0,0.82}

\title{Tessellations of Semi-Discrete Flow Matching}
\author{
Emile Pierret\footnotemark[1] \\
DMA, ENS-PSL, MAP5 \\
\texttt{pierret@math.cnrs.fr} 
\and
Johannes Hertrich\footnotemark[1] \\
DMA, ENS-PSL \\
\texttt{johannes.hertrich@ens.fr} 
\and
\hspace{-1.5cm} Samuel Hurault\footnotemark[1] \\
\hspace{-1.5cm} LIGM, Univ. Gustave Eiffel, CNRS \\
\hspace{-1.5cm} \texttt{samuel.hurault@univ-eiffel.fr} 
\and 
Julie Delon \\
DMA, ENS-PSL  \\
\texttt{julie.delon@ens.fr} 
}

\date{}

\begin{document}

\maketitle
\footnotetext[1]{Equal contribution.}

\begin{abstract}
We study Flow Matching in a semi-discrete setting where a Gaussian source is transported toward a discrete target supported on finitely many points. This semi-discrete regime is the theoretical setting behind the use of Flow Matching for generative modeling, where the target distribution is represented by a finite dataset.
In this semi-discrete regime, the exact Flow Matching velocity field is available in closed form, which makes it possible to analyze the geometry induced by the terminal flow map independently of optimization and approximation effects. We investigate the terminal assignment regions, namely the preimages of the target atoms under the terminal flow. We show that these regions are open, simply connected and, under an additional assumption, homeomorphic to the unit ball.  At the same time, a planar four-point example shows that these cells can differ sharply from Laguerre cells arising in semi-discrete optimal transport: they may be non-convex, have curved boundaries, and exhibit different boundedness and adjacency patterns. These results clarify the geometry intrinsically induced by the exact semi-discrete Flow Matching objective before neural approximation enters the picture.
\end{abstract}

\section{Introduction}

Flow Matching (FM) and closely related regression-based formulations of generative transport appeared almost simultaneously in 2022 in three works: \cite{albergo2023building,lipman2023flow,liu2023flow}. These papers share the core idea of learning a continuous-time transport by regressing a velocity field along a prescribed interpolation path, thereby avoiding backpropagation through an ODE solver during training. Since then, FM has emerged as a powerful framework for training continuous normalizing flows and as a useful conceptual bridge between diffusion-type models and deterministic generative flows~\cite{Lipman_Guide_Code_2024_Arxiv}. In particular, several recent works have emphasized the role of rectified flows and their relation to transport efficiency and straightness of trajectories~\cite{liu2023flow,bansal2025wassersteinconvergencestraightnessrectified,PSM2026}, while the precise connection between Flow Matching and optimal transport (OT) has only started to be clarified more recently~\cite{hertrich2025relationrectifiedflowsoptimal}.

Among the probability paths that can be used in FM, OT-inspired paths play a particularly important role. They often lead to faster sampling and improved training behavior~\cite{lipman2023flow,liu2022rectified,liu2023flow,Zhang:2025ab,pooladian,tong}. At the same time, recent results show that the relation between FM or rectified flows and OT is more subtle than early analogies might suggest, and that apparent similarities can break down outside rather restrictive regimes~\cite{hertrich2025relationrectifiedflowsoptimal,lavenantflow}. This raises a basic question: what geometry is intrinsically induced by the exact Flow Matching vector field itself, before any neural approximation is introduced?

We address this question in a semi-discrete setting where the source is Gaussian and the target is a finitely supported measure. This regime is appealing for two reasons. First, semi-discrete OT provides a classical reference geometry through Laguerre cells, whose geometry underlies both the theory and algorithms of semi-discrete transport~\cite{aurenhammer1998minkowski,merigot2011multiscale,kitagawa2019convergence,peyre2019computational}. Second, when one starts from the independent coupling, the exact  velocity field of semi-discrete FM is available in closed form~\cite{bertrand2026closedformflowmatchinggeneralization}, making the model analytically tractable while retaining genuinely nontrivial geometric behavior.

Our objects of study are the terminal assignment regions
\begin{equation}\label{eq:ass_regions}
\Gamma_k \coloneqq \{x\in\R^d : \gamma(x)=a_k\},
\end{equation}
where $\gamma$ is the terminal flow map and $(a_k)_{1\leq k\leq n}$ are the support points of the target measure. These sets are the exact FM analogue of semi-discrete transport cells. A natural first guess would be that they inherit at least part of the usual OT geometry, for instance convexity, affine boundaries, or adjacency patterns comparable to those of Laguerre diagrams. 
Figure~\ref{fig:introfig} already suggests that this intuition is misleading. While FM cells appear topologically simple (connected and without holes), they can be visibly non-convex, separated by curved boundaries and organized in a way that does not completely match the boundedness or neighborhood structure of the corresponding Laguerre cells.

\paragraph{Contributions} In this paper, we make the observations from Figure~\ref{fig:introfig} formal. 
First, we present some simplification results regarding the dynamics of semi-discrete FM. 
Second, as our first main contribution, we establish global topological properties of the FM cells, showing that the sets $\Gamma_k$ are open and simply connected, and under an additional assumption we prove that they are homeomorphic to the unit ball. 
Third, for our second main contribution, we exhibit a simple planar four-point example showing that if FM cells may initially look like Laguerre cells, they can differ qualitatively from them, inducing a genuinely different geometry. Specifically, we show that they may be non-convex, have curved boundaries, and display different boundedness and adjacency patterns. 



\newlength{\threeimagesfigure}
\setlength{\threeimagesfigure}{0.32\linewidth}

\begin{figure}
    \centering
    \begin{subfigure}{\threeimagesfigure}
    \centering
    \includegraphics[width=\linewidth]{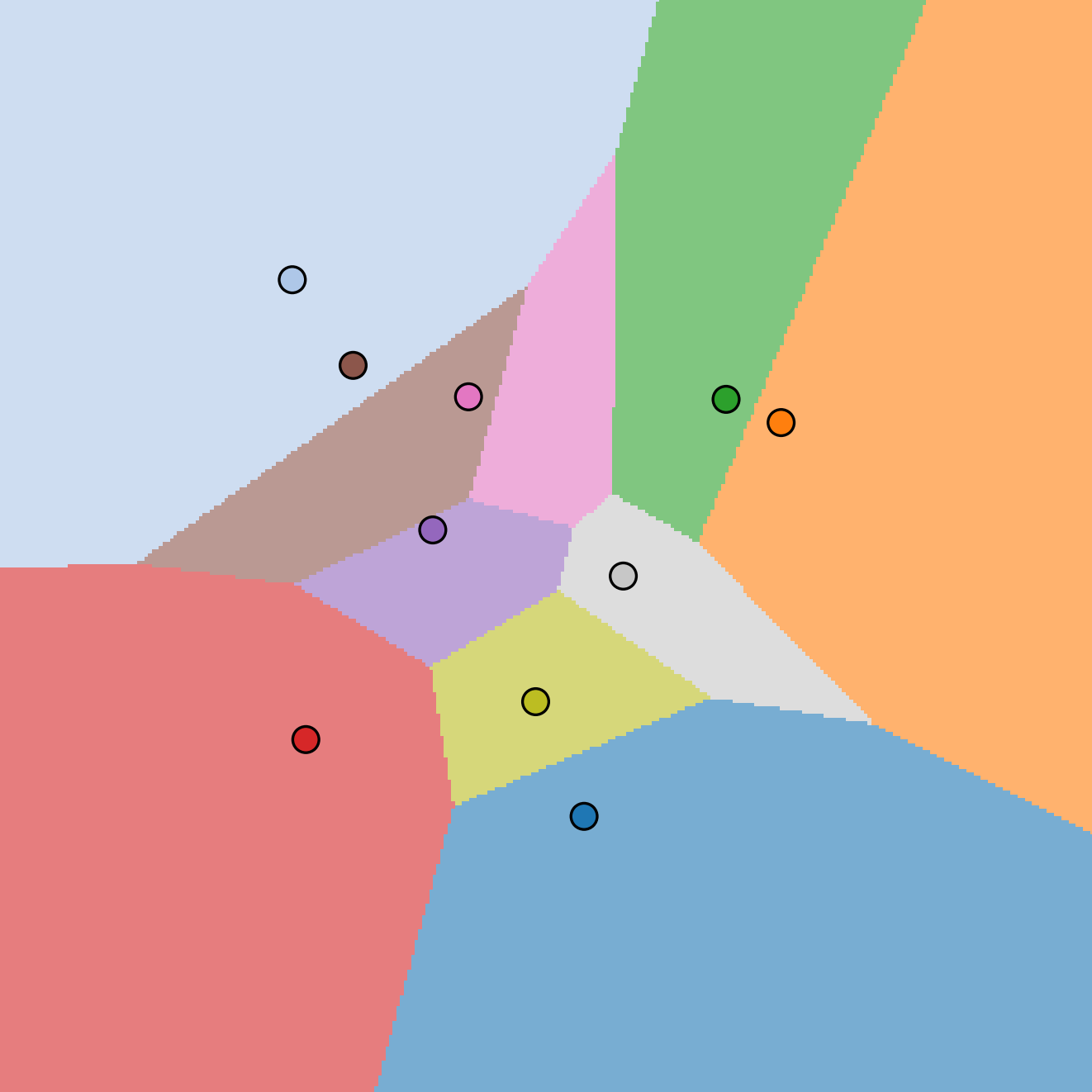}
     \caption{OT}
     \end{subfigure}
    \begin{subfigure}{\threeimagesfigure}
    \centering
    \includegraphics[width=\linewidth]{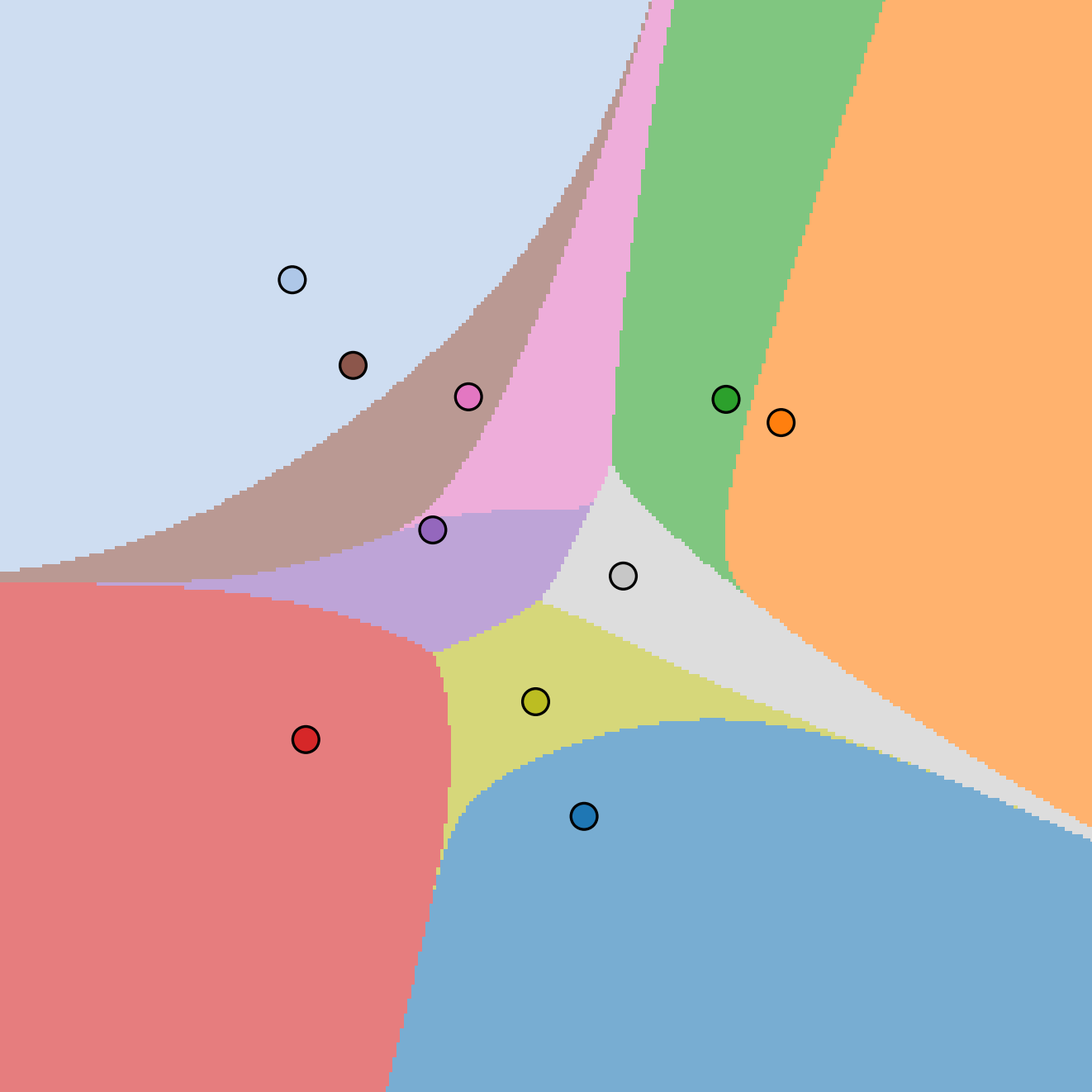}
     \caption{FM with exact velocity}
     \end{subfigure}
     \begin{subfigure}{\threeimagesfigure}
    \centering
    \includegraphics[width=\linewidth]{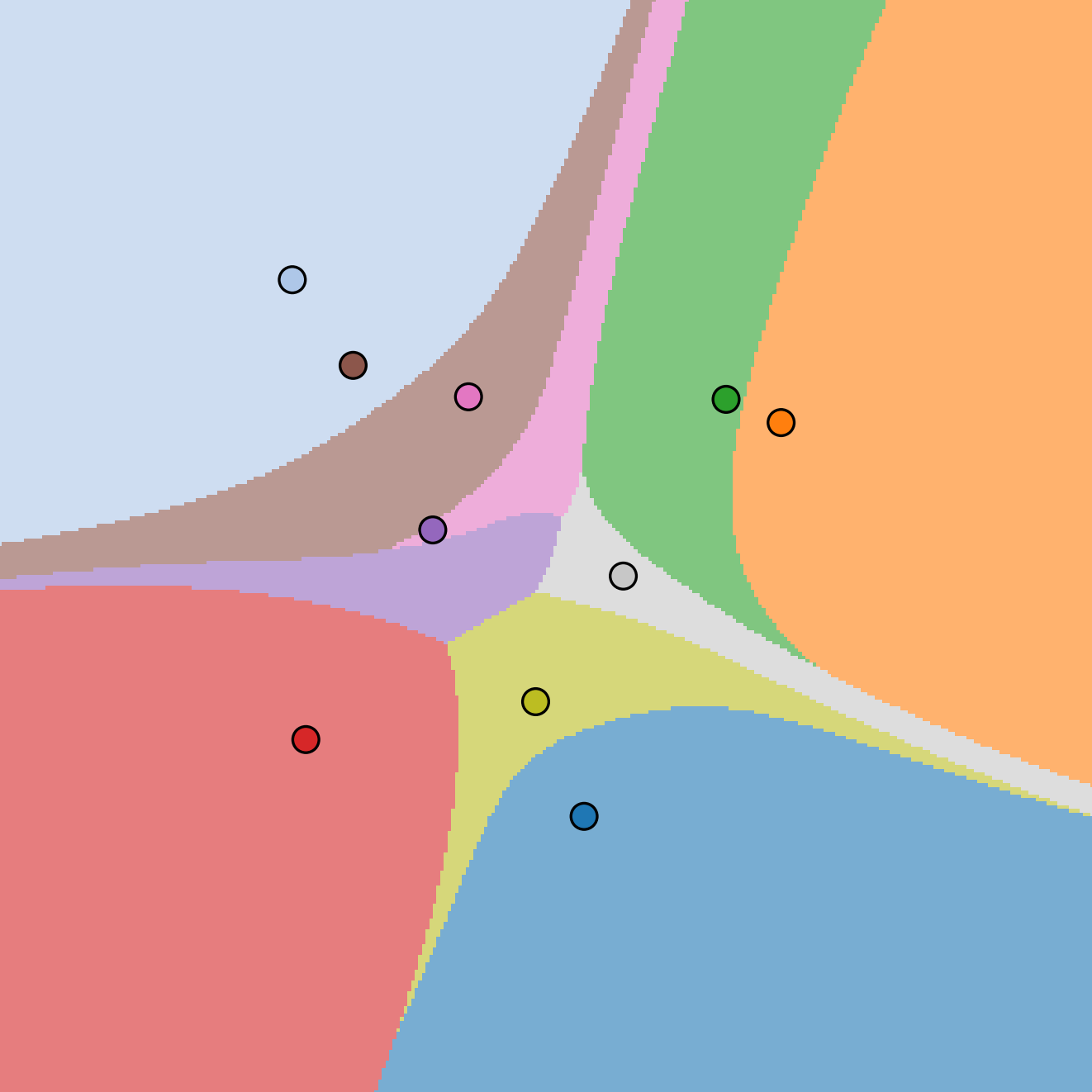}
     \caption{FM with learned velocity}
     \end{subfigure}

     \caption{\label{fig:introfig} Semi-discrete assignment cells $\Gamma_k$ for $10$ points in dimension $2$, for (a) Optimal Transport, (b) Flow Matching with closed-form velocity and (c) Flow Matching with velocity approximated by a neural network. In this last case, each grid point is assigned the color of the target point closest to its image under the terminal flow. }
\end{figure}

This perspective is relevant for generative modeling even though practical neural FM models only learn an approximation of the exact field studied here.
In discrete or nearly discrete regimes these approximations  may smooth, relax, or altogether blur the hard terminal assignments present in the exact dynamics. Our point is therefore not that trained networks should literally reproduce the cells $\Gamma_k$, although the right part of Figure~\ref{fig:introfig} suggests that these cells are still visible in the trained velocity field, at least in small dimension.  
Rather, these cells reveal the canonical geometric bias encoded in the FM objective before parameterization, optimization, and approximation effects enter the picture. Understanding this reference geometry is a first step toward disentangling what is intrinsic to the objective itself from what is induced by approximation, architecture, or optimization.


\section{Background on Flow Matching and Semi-discrete Optimal Transport}

\paragraph{Basics of Flow Matching}
We briefly recall the Flow Matching (FM)  principle before specializing to the discrete target setting considered in this paper. Flow Matching connects a source distribution $\mu_0$ and a target distribution $\mu_1$, using a continuous sequence of probability measures $(\mu_t)_{t\in[0,1]}$. FM proposes a time-dependent velocity field $v_t$ such that the associated continuity equation transports $\mu_0$ into $\mu_1$, or equivalently such that the flow map generated by the ODE
\begin{equation}
\label{eq:general_fm_ode}
\dot \gamma_t(x)=v_t(\gamma_t(x)),
\qquad \gamma_0(x)=x,
\end{equation}
satisfies $(\gamma_t)_\#\mu_0=\mu_t$ (provided that it exists and is unique). 

A key feature of FM is that one does not need to integrate the ODE during training. Instead, one regresses the velocity field against a target conditional velocity induced by a prescribed family of conditional paths~\cite{albergo2023building,lipman2023flow,Lipman_Guide_Code_2024_Arxiv, liu2023flow}.
More precisely, if $(X_0,X_1)$ is a coupling between $\mu_0$ and $\mu_1$ and if $X_t=(1-t)X_0+tX_1$, then the target FM velocity is given by the conditional expectation
\begin{equation}
\label{eq:general_conditional_velocity}
v_t^\star(x)=\E{X_1 - X_0\mid X_t=x}.
\end{equation}
This $L^2$ conditional expectation is naturally characterized as the minimizer of a least-squares regression problem. In other words, one can train a parametric field $v_t^\theta$ by regressing it against the conditional path derivative $X_1-X_0$. Observe that in the case where $\mu_0\sim\mathcal N(0,I)$ and where $(X_0,X_1)\sim\mu_0\otimes\mu_1$ is the independent coupling, existence and uniqueness of the flow map for the velocity $v_t^\star$ is easy to show (see \Cref{sec:existence_uniqueness}). Further, the flow map $\gamma_t$ is a bi-Lipschitz diffeomorphism for $0\leq t < 1$, even though it is not invertible for $t=1$.

\paragraph{Link to Optimal Transport.} For the quadratic cost, Monge optimal transport is classically formulated as the optimization problem
\begin{equation}
\label{eq:static_ot}
\inf_{T_\#\mu_0=\mu_1}\int_{\R^d}\|x-T(x)\|^2\,d\mu_0(x).
\end{equation}
When an optimal map $T^*$ exists (which holds as soon as $\mu_0$ has a density for instance), it induces a canonical path $\mu_t:= \left((1-t)\mathrm{Id}+tT^*\right)\#\mu_0$ between $\mu_0$ and $\mu_1$. This is the basic intuition behind the link with Flow Matching: both frameworks describe transport through time-dependent probability paths and velocity fields, but they optimize different objects. In particular, Flow Matching looks for a velocity field associated with a prescribed interpolation $(\mu_t)_{t\in[0,1]}$, whereas OT selects directly the map minimizing the transport cost. The two viewpoints coincide in special cases, for instance in one dimension or for suitable Gaussian settings (for independent $X_0$ and $X_1$ with commuting covariances), but not in general~\cite{hertrich2025relationrectifiedflowsoptimal,liu2022rectified,pierret2026flow}.

\paragraph{Semi-discrete case.}

We focus on the specific case where $(X_0,X_1)$ is an independent coupling with
$X_0 \sim \mu_0:=\mathcal{N}(0,\id)$
and
$X_1 \sim \mu_1:=\frac{1}{n}\sum_{k=1}^n \delta_{a_k}$,
with pairwise distinct support points $a_k\neq a_l$ for $k\neq l$.
In this situation, both the optimal transport problem and the flow matching algorithm admit more tractable formulations.

The optimal transport map (with quadratic cost) between $\mu_0$ and $\mu_1$ is characterized by so-called Laguerre cells: 
there exists a vector of weights $\psi=(\psi_k)_{1\le k\le n}$ such that it holds $T(x)=a_k$ for $x\in \mathcal{L}_k(\psi)$ almost everywhere, where
\begin{equation}\label{eq:cells_laguerre}
\mathcal{L}_k=\mathcal{L}_k(\psi)\coloneqq \left\{x\in\R^d:\ \frac12\|x-a_k\|^2-\psi_k\le \frac12\|x-a_j\|^2-\psi_j,\ \forall j\right\}.
\end{equation}
These cells form a polyhedral partition of the source space. For all weights equal, for instance when $\psi_k=0$ for every $k$, they are also known as Voronoi cells~\cite{aurenhammer1998minkowski,merigot2011multiscale,kitagawa2019convergence,peyre2019computational}.


For flow matching, the semi-discrete setting allows the derivation of a closed formula for the velocity field given by
\begin{equation}
\label{eq:velocity_Gaussian_to_Dirac}
    v_t(x) := \sum_{k=1}^n \alpha_k(t,x) \frac{a_k-x}{1-t},\quad \alpha(t,x) := \softmax\left(-\frac{1}{2(1-t)^2}\|x-ta_i\|^2\right)_{1 \leq i \leq n}.
\end{equation}
The velocity is a convex combination of the elementary drifts $(a_k-x)/(1-t)$ pointing toward the support points.
The coefficients $\alpha_k(t,x)$ can be interpreted as posterior assignment weights: at time $t$ and position $x$, they measure how likely the endpoint $a_k$ is under the conditional interpolation model.  We denote by $\gamma_t\colon\R^d\to\R^d$ the flow map generated by this field, defined by~\eqref{eq:general_fm_ode}, and {\bf we write $\gamma=\gamma_1$ for the terminal map.}
Our work studies the topology and geometry induced by this terminal flow map $\gamma$.

\section{Some Simplification Results}\label{sec:simplifications}

Before presenting our main results, we first explain how several related problems can be reduced to the framework studied in this paper. Specifically, we establish three auxiliary statements. First, we observe that scaling, translating, or rotating the cells under consideration leaves our problem unchanged. Second, we prove that it suffices to restrict our analysis to the affine hull of the points $a_k$, and that the dynamics outside this hull are trivial. Third, we demonstrate that discrete FM arises as a limiting case of Gaussian mixture targets when the variances of the modes tend to zero.

\paragraph{Invariance under Scalings, Shifts and Rotations}

We note that the cells are invariant under scalings and shifts and equivariant under orthogonal transformations. More precisely, we have the following corollary, which directly follows from \cite[Thm 2]{hertrich2025relationrectifiedflowsoptimal}.
\begin{cor}
\label{cor:scaling}
Let $A\in\R^d\times\R^d$ be an orthogonal matrix, $b\in\R^d$ and $c>0$. Let $\gamma$ be the flow map for terminal points $a_1,...,a_n\in\R^d$ and $\tilde \gamma$ be the flow map for terminal points $\tilde a_1,...,\tilde a_n$ defined by $\tilde a_k=cAa_k+b$. Then it holds $\tilde\gamma(x)=cA\gamma(A^{-1}x)+b$ and $\tilde\Gamma_k=\{Ax:x\in\Gamma_k\}$, where $\Gamma_k$ and $\tilde \Gamma_k$ are the assignment regions \eqref{eq:ass_regions} corresponding to $\gamma$ and $\tilde\gamma$.
\end{cor}

This corollary is illustrated in Figure~\ref{fig:cor_scaling}. In particular, as in the case of Laguerre cells in optimal transport, a point may fail to belong to its associated cell.

\begin{figure}[h]
    \centering
    \begin{subfigure}{\threeimagesfigure}
    \centering
    \includegraphics[width=\linewidth]{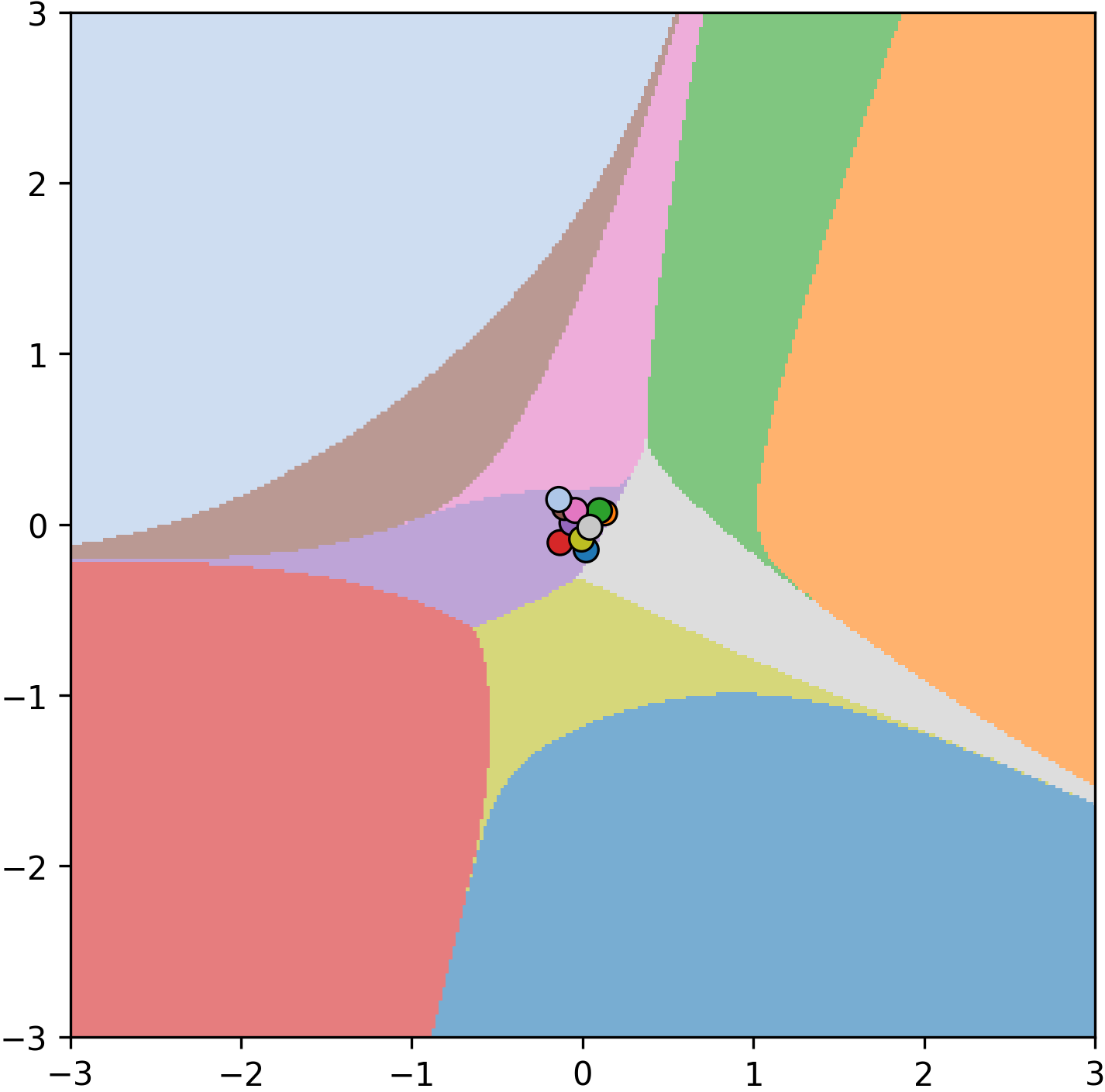}
     \caption{$c = 0.1$}
     \end{subfigure}%
    \begin{subfigure}{\threeimagesfigure}
    \centering
    \includegraphics[width=\linewidth]{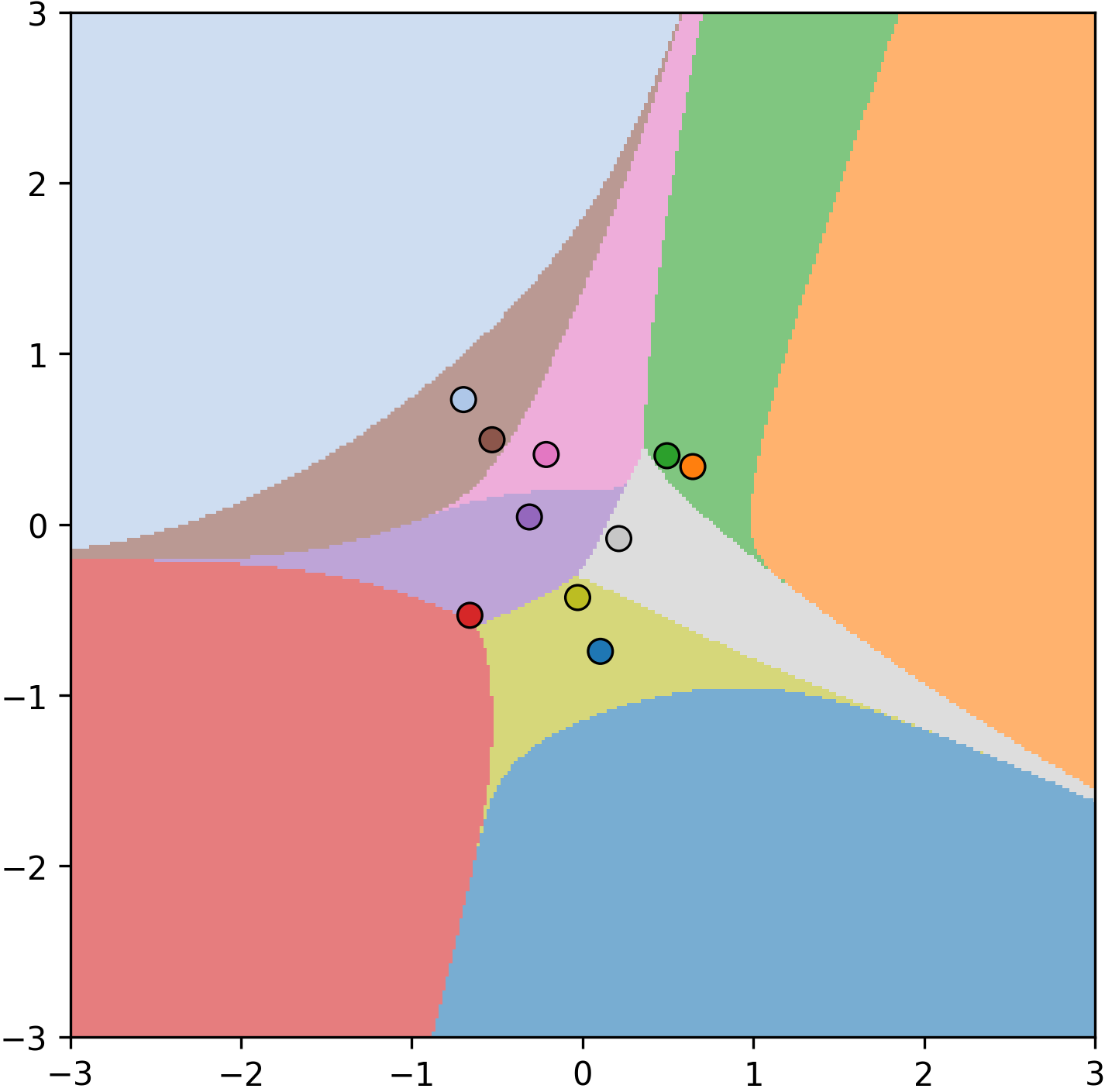}
     \caption{$c = 0.5$}
     \end{subfigure}%
     \begin{subfigure}{\threeimagesfigure}
    \centering
    \includegraphics[width=\linewidth]{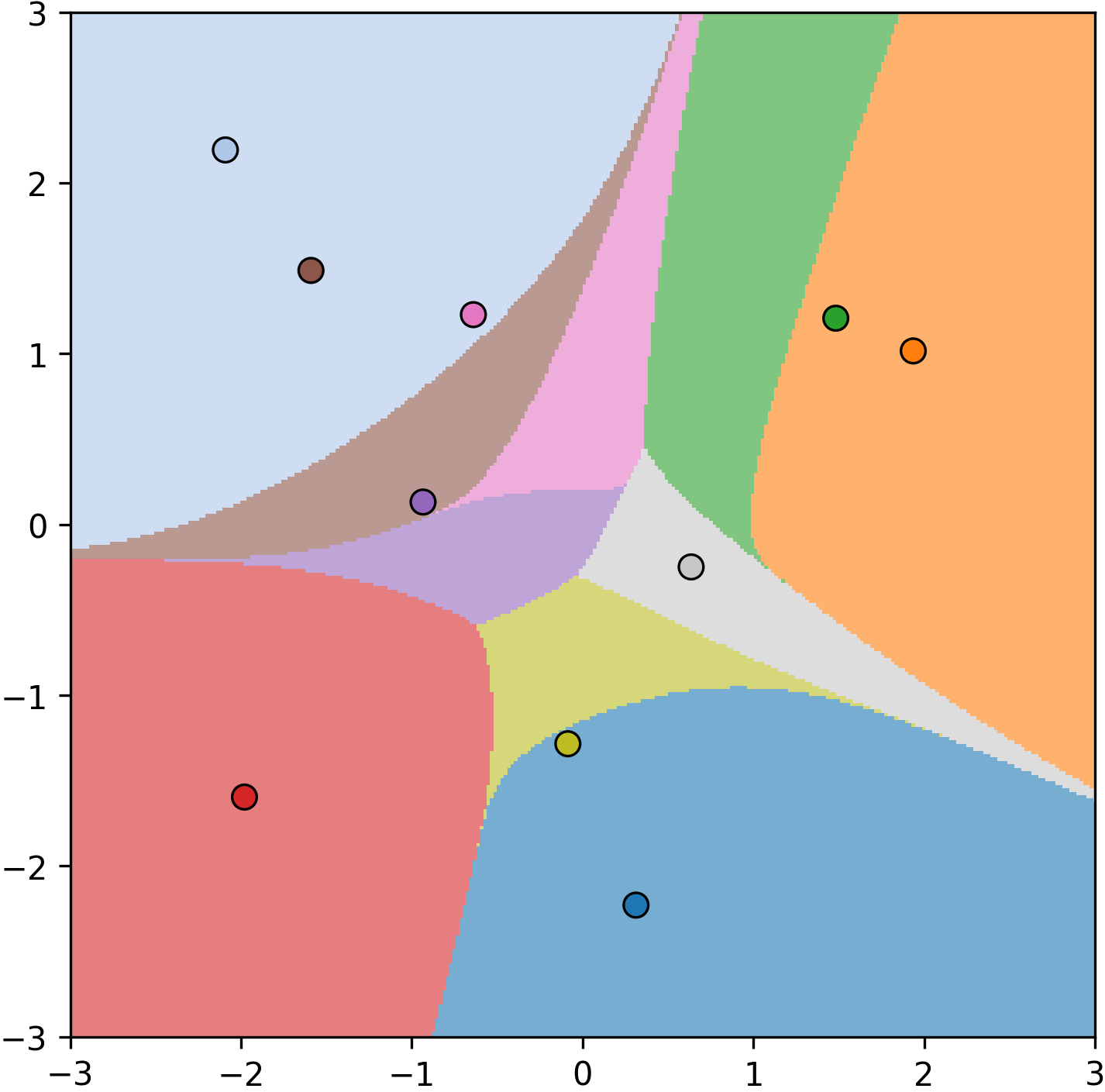}
     \caption{$c = 1.5$}
     \end{subfigure}

     \caption{\label{fig:cor_scaling} Illustration of \Cref{cor:scaling}. For \(n = 10\) points \((a_k)_{1 \leq k \leq n}\), we display the associated semi-discrete flow matching cells of \((c a_k)_{1 \leq k \leq n}\) for \(c = 0.1, 0.5,\) and \(1.5\), which coincide. This also shows that a point may or may not belong to its associated cell. }
\end{figure}


\paragraph{Reduction to the Affine Hull}

Next, we show that the FM dynamics can be completely reduced to the dynamics on the affine hull of the points $a_k$. More precisely, let
\[
E\coloneqq \mathrm{Aff}(a_k)_{1\le k\le n}=a_1+\Span\{a_k-a_1:k=2,...,n\},
\]
and decompose $x=x^\parallel+x^\perp\in\R^d$ into the orthogonal projection $x^\parallel\in E$ and its complement $x^\perp=x-x^\parallel$.
Now the next proposition shows that the dynamics $\gamma_t(x)$ can be decomposed into the dynamics of $x^\parallel$ and a linear shift towards the affine space $E$.
The proof is given in \Cref{proof:span_reduction}.

\begin{prop}
\label{prop:only_depends_on_span}
For every $t<1$ and $x=x^\parallel+x^\perp\in\R^d$, one has
$
\alpha_k(t,x)=\alpha_k(t,x^\parallel)
$ for all $k$ and $\gamma_t(x)=\gamma_t(x^\parallel)+(1-t)x^\perp$.
\end{prop}

This reduction shows in particular that genuinely multidimensional phenomena can only arise when the support points span a space of dimension at least two. In the collinear case, the problem reduces to one dimension, which explains why the comparison with semi-discrete OT is especially natural in that regime~\cite{hertrich2025relationrectifiedflowsoptimal,pierret2026flow}.



\paragraph{Semi-discrete Flow Matching as Limit of Gaussian Mixtures}

Next, we show that semi-discrete FM can be understood as the limit case of FM with Gaussian mixture target for vanishing variance.
More precisely, let $\mu_1^\varepsilon=\frac1n\sum_i \mathcal{N}(a_i, \varepsilon^2I)$ and denote the corresponding velocity field from the FM dynamics by 
{\small
\begin{equation}
\label{eq:velocity_GMM}
    v_t^\varepsilon(x) := \sum_{k=1}^n \alpha_k^\varepsilon(t,x) \frac{(1-t)(a_k-x) + t \varepsilon^2 x}{(1-t)^2 + t^2 \varepsilon^2}, \quad 
    \alpha^\varepsilon(t,x) := \softmax\left(-\frac{\|x-ta_i\|^2}{2\big((1-t)^2 + t^2 \varepsilon^2 \big)}\right)_{1 \leq i \leq n}.
\end{equation}}%
We denote by $\gamma^\varepsilon_t \colon\R^d\to\R^d$ the flow map defined by $\gamma^\varepsilon_0(x)=x$ and $\dot\gamma^\varepsilon_t(x)=v^\varepsilon_t(\gamma^\varepsilon_t(x))$ and use the notation $\gamma^\varepsilon=\gamma^\varepsilon_1$.
For $\varepsilon=0$, we obtain the velocity $v_t=v_t^0$ and flow map $\gamma_t=\gamma_t^0$ of our semi-discrete case.
We then have the following pointwise convergence result for \(\gamma^\varepsilon\) toward \(\gamma\) as \(\varepsilon\to0\), valid inside the limiting cells.
The proof is given in \Cref{app:GMM_0}.
\begin{prop} \label{lem:GMM_0}
Let $x \in \cup_{k} \Gamma_k$. For all $\delta > 0$, there is $\varepsilon_0 > 0$ and $T \in (0,1)$ such that for all $\varepsilon < \varepsilon_0$ and $t \in [T, 1]$ it holds
$
    \norm{\gamma^\varepsilon_t(x) - \gamma_1(x)} \leq \delta
$.
In particular, for $t=1$, $\gamma^\varepsilon(x)\to\gamma(x)$ as $\varepsilon\to 0$.
\end{prop}

\newlength{\fourimagesfigure}
\setlength{\fourimagesfigure}{0.22\linewidth}
We illustrate on Figure~\ref{fig:eps_conv} the convergence from the proposition by considering the generalized assignment regions 
\begin{equation}\label{eq:gen_assignment}
\Gamma^\varepsilon_k=\{x\in\R^d:\|\gamma^\varepsilon(x)-a_k\|<\|\gamma^\varepsilon(x)-a_l\|\text{ for }l\neq k\}.
\end{equation}
For $\varepsilon=0$, these regions coincide with the assignment regions $\Gamma_k$ from \eqref{eq:ass_regions} and the above Proposition implies, for every $x\in \bigcup_{j=1}^n \Gamma_j$,
$
    \mathbf{1}_{\Gamma_k^\varepsilon}(x)
    \to
    \mathbf{1}_{\Gamma_k}(x)
$ (see Corollary~\ref{cor:cell_convergence} in ~\cref{app:GMM_0}).
For $\varepsilon>0$, we find that $\gamma^\varepsilon$ is a bi-Lipschitz diffeomorphism such that the $\Gamma_k^\varepsilon$ and the corresponding neighborhood relations are topologically equivalent to the Voronoi cells.

\begin{figure}[h]
    \centering
    \begin{subfigure}{\fourimagesfigure}
    \centering
    \includegraphics[width=0.9\linewidth]{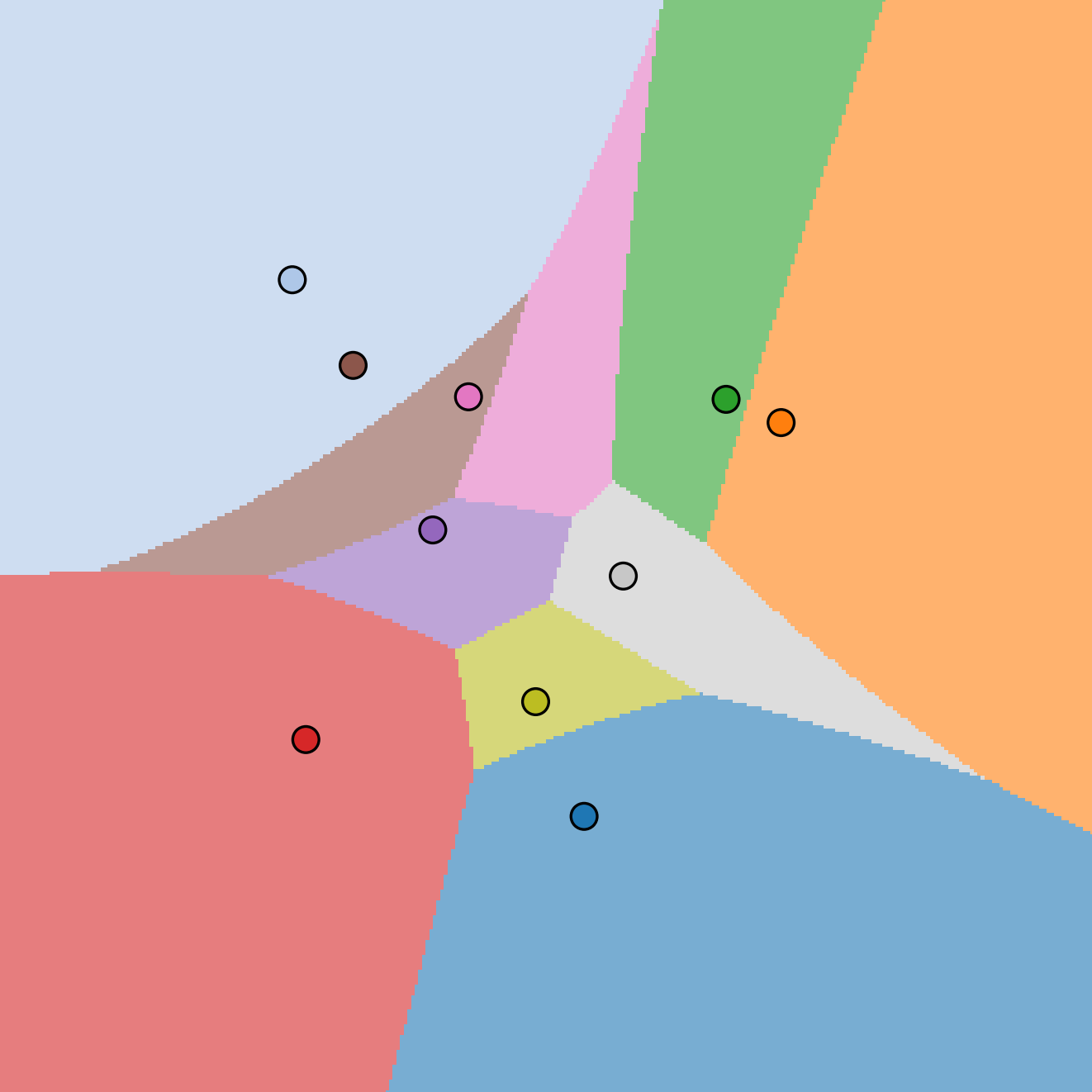}
     \caption{$\varepsilon = 0.75$}
     \end{subfigure}
    \begin{subfigure}{\fourimagesfigure}
    \centering
    \includegraphics[width=0.9\linewidth]{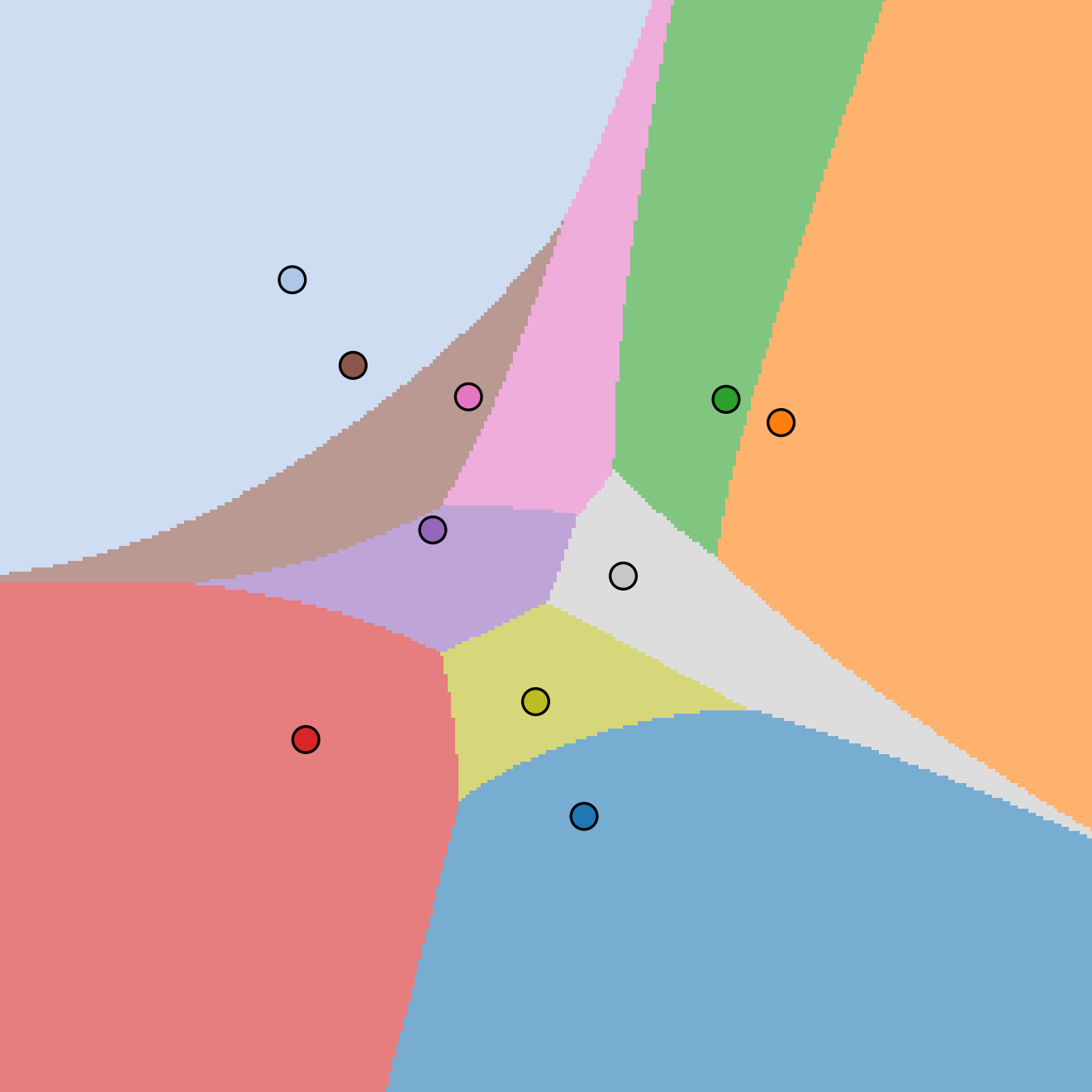}
     \caption{$\varepsilon = 0.5$}
     \end{subfigure}
     \begin{subfigure}{\fourimagesfigure}
    \centering
    \includegraphics[width=0.9\linewidth]{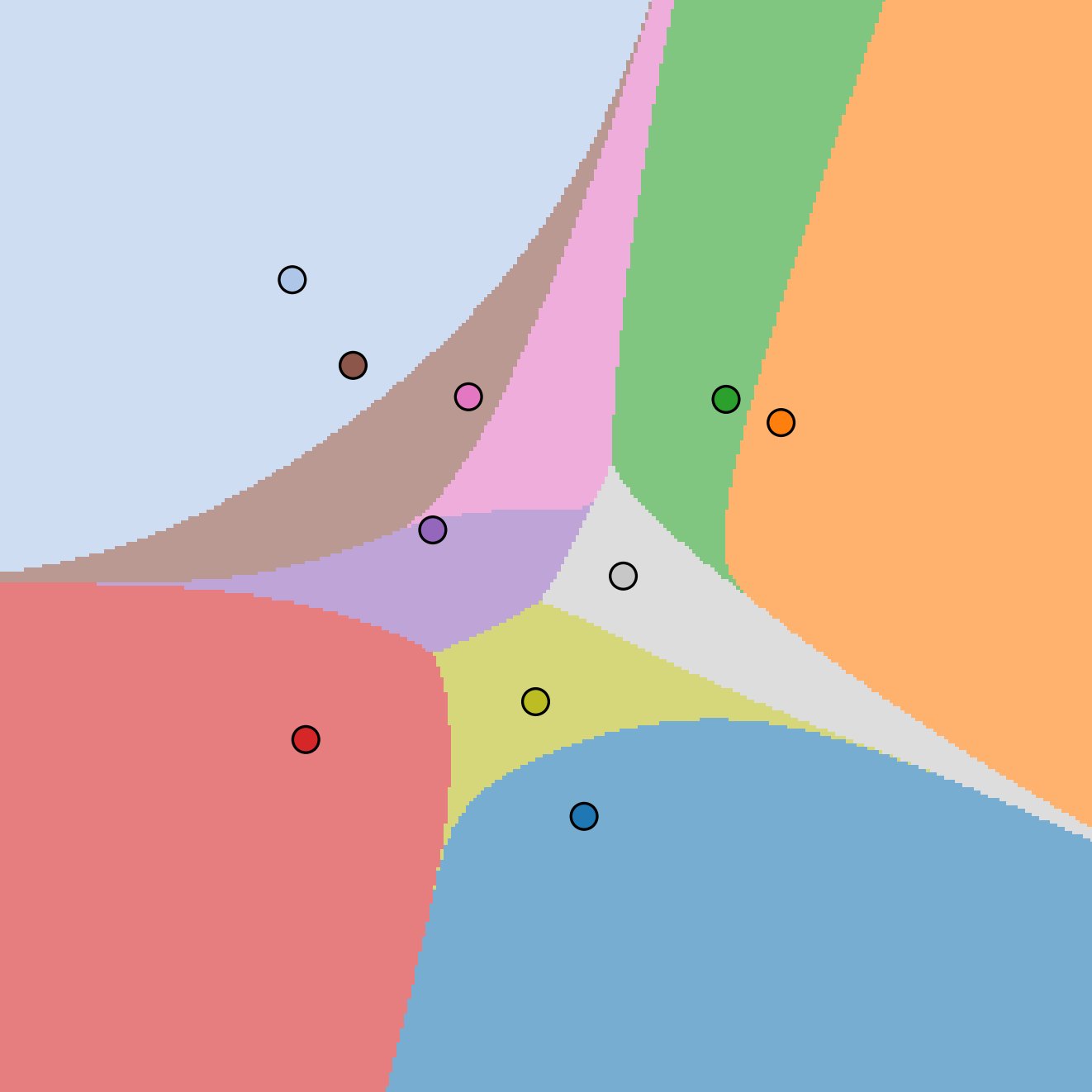}
     \caption{$\varepsilon = 0.2$}
     \end{subfigure}
     \begin{subfigure}{\fourimagesfigure}
    \centering
    \includegraphics[width=0.9\linewidth]{figures/intro/seed_14/trajectory_endpoints_comparison_10_points_true.png}
     \caption{$\varepsilon = 0$}
     \end{subfigure}
     \caption{We plot the generalized assignment regions \eqref{eq:gen_assignment} for different values of $\epsilon$. According to \Cref{lem:GMM_0} these regions approximate the $\Gamma_k$ from \eqref{eq:ass_regions} as $\epsilon\to 0$.}
     \label{fig:eps_conv}
\end{figure}

\section{Topology of the Flow Matching Cells}

We now turn our attention to the topology of the assignment regions $\Gamma_k$. Informally, these regions specify which point in the source distribution is associated with which point in the target distribution. In this sense, the regions $\Gamma_k$ play a role analogous to Laguerre cells, except that the optimal transport coupling is replaced by the coupling induced by flow matching. For this reason, we will also refer to the $\Gamma_k$ as \emph{flow matching cells}.

Although we have already observed numerically in the introduction (Figure~\ref{fig:introfig}) that the cells $\Gamma_k$ do not share the neighborhood relations and boundedness properties of the Laguerre cells, we demonstrate in this section that, from a topological perspective, they are nevertheless well-behaved: each $\Gamma_k$ is an open, connected set without holes. 
Under an additional assumption, we will further prove that these cells are in fact topologically equivalent (homeomorphic) to the unit ball.
The key ingredient for establishing these results is a late-time capture and stability result which we describe in the next subsection.




\subsection{Late-Time Capture and Stability}

Given that the flow map $\gamma_t$ is a bi-Lipschitz diffeomorphism for $0\leq t<1$ but not even invertible for time $t=1$, it will be important to understand in which sense $\gamma=\gamma_1$ is approximated by the maps $\gamma_t$ for $0\leq t < 1$. While it is directly clear from the existence of the flow map (see also \Cref{sec:existence_uniqueness}) that $\gamma_t$ converges to $\gamma$ almost everywhere, the next proposition shows a stronger convergence result.
More precisely, we show that any trajectory $t\mapsto\gamma_t(x)$ which enters a ball around an atom $a_k$ late enough convergences to the atom.

\begin{prop}
\label{prop:into_the_ball_implies_convergence_to_ak}
	Define $r<\frac12 \min_{k\neq l}\|a_k-a_l\|$. Then, there exists $\underline{t}_r\in(0,1)$ such that for $x \in \R^d$, if there exists $\underline{t}_r < t < 1$ such that $\gamma_{t}(x) \in  \bar B_r(a_k)$ then $\gamma_1(x) = a_k$.
\end{prop}

The proof is deferred to \Cref{proof:into_the_ball_implies_convergence_to_ak} and relies on \Cref{lem:velocity_points_inwards}, which states that close enough to time~$1$, the velocity points inward on a small sphere around each support point.
We also note that a similar result in the context of the probability flow of diffusion models was proven by \cite{baptista2025memorization}.


\subsection{Cell connectedness}


Next, we prove that the flow matching cells $\Gamma_k$ are open and do not admit any holes.
We formalize this property by showing in the next theorem that the cells $\Gamma_k$ are simply connected. Recall that $A\subseteq\R^d$ (viewed as a topological space inheriting the topology of $(\R^d,\|\cdot\|)$) is simply connected if for any loop $f\colon S^1 \to A$ we can find a continuous mapping (homotopy) $H\colon[0,1]\times S^1\to A$ such that $H(0,x)=f(x)$ and $H(1,\cdot)$ is constant.

\begin{thm}
\label{cor:Gamma_k_connectedness}
    For $1 \leq k \leq n$, the set $\Gamma_k$ is open and simply connected.
\end{thm}

The detailed proof is provided in \Cref{sec:proofs_topological}. We outline the key arguments as follows: The openness of $\Gamma_k$ is an immediate consequence of the fact that it can be written as a union of open sets of the form $\gamma_t^{-1}(B_r(a_k))$, where $r$ comes from \Cref{prop:into_the_ball_implies_convergence_to_ak} and $t$ is sufficiently large. Using this open cover of $\Gamma_k$, we further show that for every compact subset $A\subseteq \Gamma_k$ there exists a time $t_0$ such that $\gamma_t(A)\subseteq B_r(a_k)$ for all $t\geq t_0$. Applying this to $f(S^1)$ yields the desired homotopy $H$: we first map by $\gamma_{t_0}$, then contract inside the ball $B_r(a_k)$, and finally apply the inverse of $\gamma_{t_0}$. By construction $H$ is continuous, and by \Cref{prop:into_the_ball_implies_convergence_to_ak} it additionally satisfies $H(t,x)\in\Gamma_k$. Altogether, this shows that $\Gamma_k$ is simply connected.

\subsection{\textit{Center} of the cells and contractibility}

\begin{wrapfigure}{r}{0.4\textwidth}
\vspace{-1cm}
    \centering
    \includegraphics[width=\threeimagesfigure]{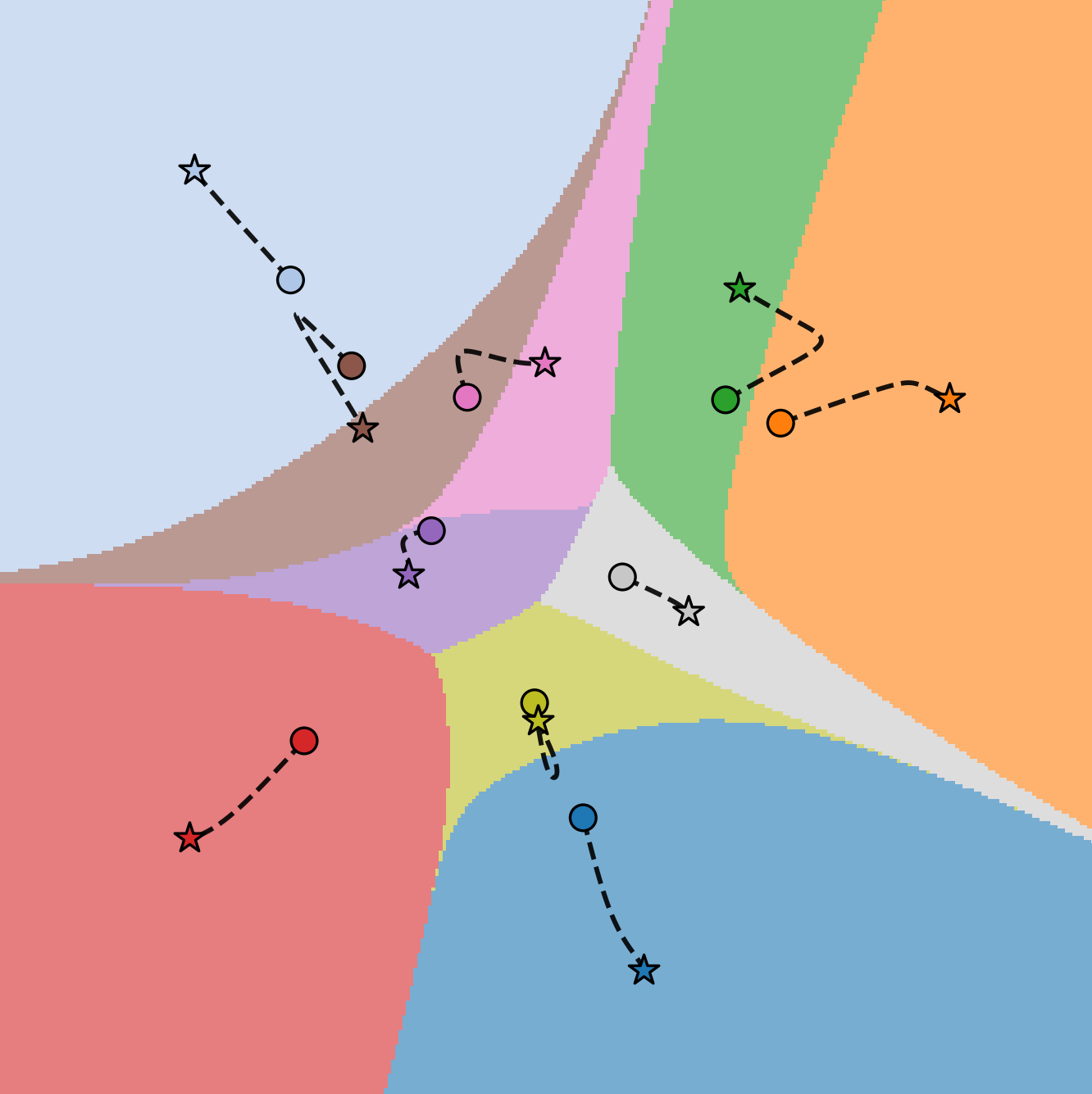}
    \caption{\label{fig:illus_centers} Illustration of the \textit{centers} of the cells. 
For \(n = 10\) points \((a_k)_{1 \leq k \leq n}\), we depict the associated semi-discrete flow matching cells and the corresponding \textit{centers} \(\lim_{t \to 1} \gamma_t^{-1}(a_k)\), represented by star symbols. 
The dotted lines indicate the curves \((\gamma_t^{-1}(a_k))_{0 \leq t < 1}\).
}
\vspace{-15pt}
\end{wrapfigure}

We have seen in the previous subsection that the cells $\Gamma_k$ are simply connected, meaning they do not contain ``one-dimensional'' holes. In order to prove the absence of higher dimensional holes, we use the notion of contractibility. Recall that $A\subseteq\R^d$ is contractible if there exists a continuous mapping (homotopy) $H\colon[0,1]\times A\to A$ such that $H(0,x)=x$ and $H(1,\cdot)$ is constant. Naturally, this property implies simply connectedness, but not vice versa.

To construct such a homotopy for the sets $\Gamma_k$, we again make use of \Cref{prop:into_the_ball_implies_convergence_to_ak}. We define $H(t,\cdot)$ by first applying $\gamma_t$, then (for sufficiently large time) contracting inside the ball $B_r(a_k)$ towards $a_k$, and finally applying $\gamma_t^{-1}$. Intuitively, as $t\to 1$, this converges to the point $\lim_{t\to 1}\gamma_t^{-1}(a_k)$. It is not immediately obvious that this limit exists and lies inside the cell $\Gamma_k$; this is established in the next proposition. The proof is quite technical and deferred to \Cref{appendix:zentren}. Since this limit always lies in (the closure of) the cell and we contract the cell onto it, we refer to it as the \textit{center} of the cell.

\begin{prop}
The curve $t\mapsto\gamma_t^{-1}(a_k)$ has finite length for $t\in(0,1)$. In particular, the limit $\lim_{t\to1}\gamma_t^{-1}(a_k)$ exists and is located in $\overline{\Gamma_k}$.
\end{prop}

We plot the curves $t\mapsto\gamma_t^{-1}(a_k)$ in \Cref{fig:illus_centers}. By definition, they start in $a_k$ for $t=0$ and end at $\lim_{t\to1}\gamma_t^{-1}(a_k)$, which is always inside $\overline{\Gamma_k}$ even when $a_k\not\in \overline{\Gamma_k}$.

Assuming that the limit $\lim_{t\to1}\gamma_t^{-1}(a_k)$ lies strictly in the interior of the cell (and not on its boundary), we can construct a homotopy as outlined above, which shows that $\Gamma_k$ is indeed contractible. In fact, under this assumption the cells turn out to be homeomorphic to $B_1(0)$, i.e., they are topologically equivalent to a ball. We summarize this in the following theorem. The explicit definition of the homotopy (together with a proof of its continuity) is given in \Cref{appendix:zentren}.

\begin{thm}
\label{prop:contractible}
    If $\lim_{t\to1}\gamma_t^{-1}(a_k)\in\Gamma_k$, then $\Gamma_k$ is contractible and homeomorphic to the open ball centered at \(0\) with radius \(1\), \(B_1(0)\) (and hence also to \(\mathbb{R}^d\)).
\end{thm}






\section{Geometry of the Flow Matching Cells}

Next, we are interested in geometric properties of the flow matching cells $\Gamma_k$. In particular, we want to examine geometric properties, like non-convexity and boundedness, and the neighborhood relations between the cells. Our results in this section will primarily be negative. In particular, we will demonstrate on a simple example with four atoms in the $\R^2$ that the cells are generally non-convex and that their boundedness properties and neighborhood relations sharply differ from the Laguerre cells.

Afterwards, we are interested whether the flow map generated by the semi-discrete flow matching is monotone, which is for particular interest in the context of reflow \cite{liu2023flow}. More precisely, \cite{bansal2025wassersteinconvergencestraightnessrectified, PSM2026} present monotonicity as a sufficient criterion for the second reflow to be a straight-line flow. Also in this context we will find a negative result and show that monotonicity is violated for our four-points example as well as for a smoothed version with Gaussian mixtures.


\subsection{A Four-Point Counterexample for Geometric Comparison with Semi-Discrete OT}
\label{sec:counter}

In the following, we provide an example in which we can explicitly prove that the geometric structure of the flow matching cells differs from the Laguerre cells obtained from OT.
More specifically, we study the setting of $n=4$ points in $\R^2$, where the points $a_k$ are defined by $a_k=(\cos(\frac{2k\pi}{3}),\sin(\frac{2k\pi}{3}))$ for $k=1,2,3$ and $a_4=(0,0)$.
The corresponding Laguerre cells and flow matching cells are illustrated in \Cref{fig:counterex} (a) and (b). 

\begin{figure}[h]
    \centering
    \begin{subfigure}{\threeimagesfigure}
    \centering
    \includegraphics[width=0.9\linewidth]{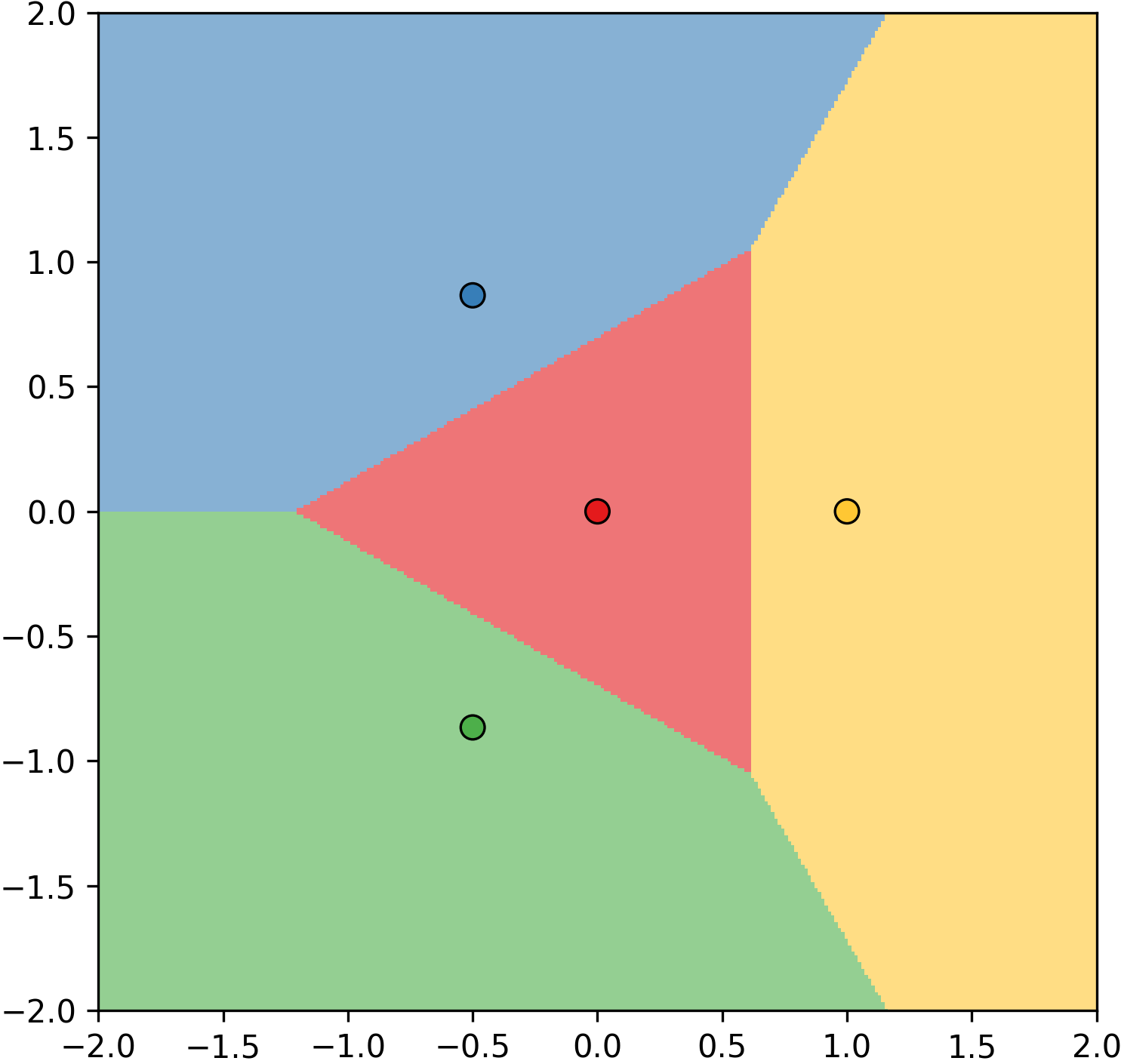}
     \caption{OT}
     \end{subfigure}%
    \begin{subfigure}{\threeimagesfigure}
    \centering
    \includegraphics[width=0.9\linewidth]{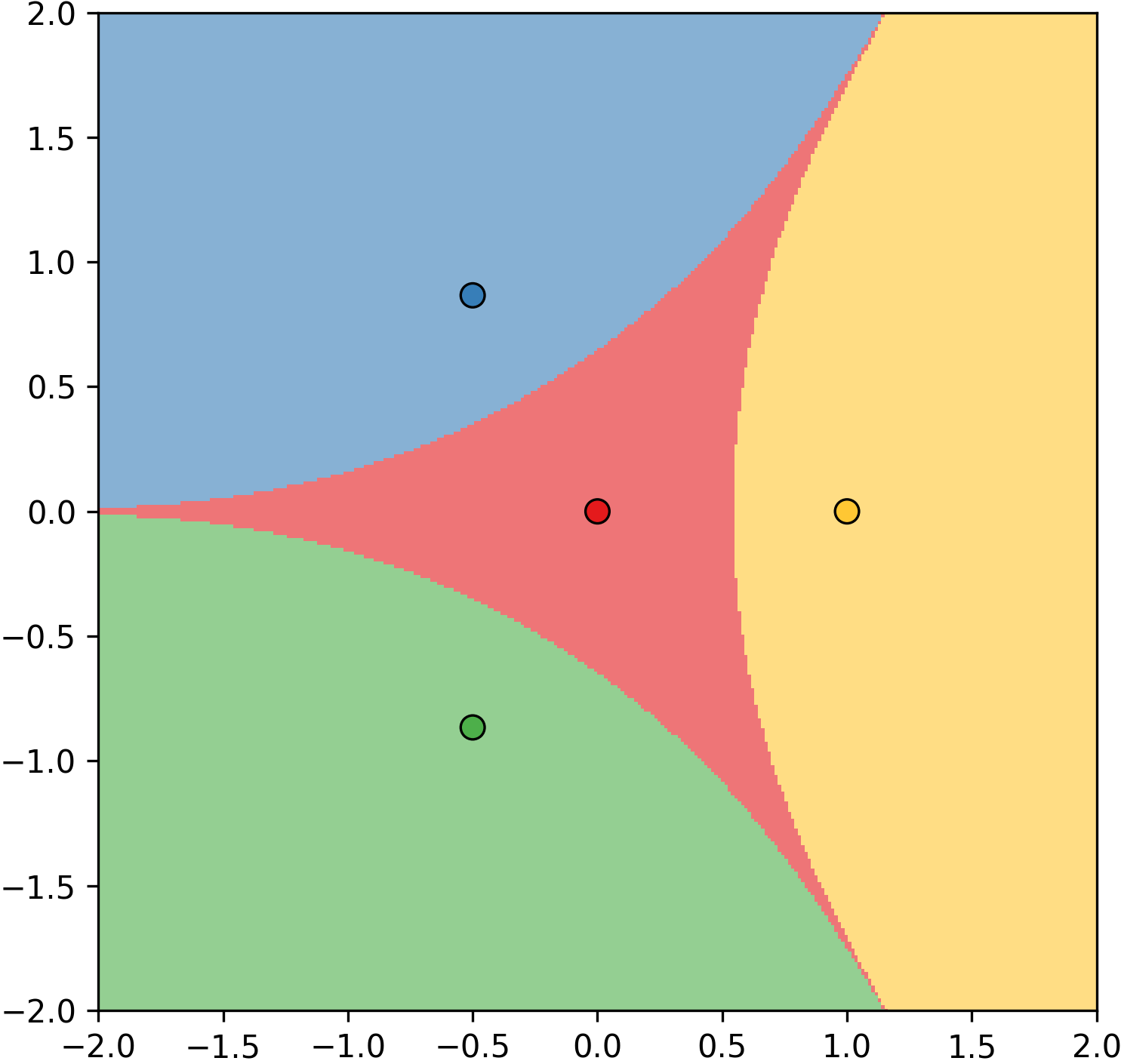}
     \caption{FM}
     \end{subfigure}%
     \begin{subfigure}{\threeimagesfigure}
     \centering
     \includegraphics[width=0.8\linewidth]{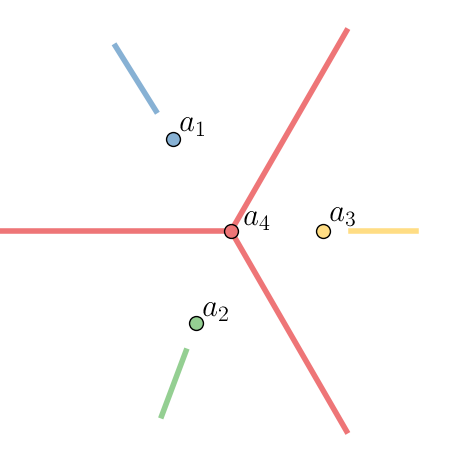}
     \caption{Illustration of \Cref{lem1:regular_triangle_example,lem2:regular_triangle_example}}
     \end{subfigure}%
     \caption{\label{fig:four_points_counterex} Illustration of the counterexample analyzed in \Cref{sec:counter} With this particular configuration of points $a_k$, we show the numerical computation of (a) the Optimal Transport (OT) Laguerre cells (b) the Flow Matching (FM) cells. (c) Visualization of the lines that are proved to be contained in the corresponding cells, with each line shown in the color of its cell (see \Cref{lem1:regular_triangle_example} and \ref{lem2:regular_triangle_example}).}
     \label{fig:counterex}
\end{figure}

In particular, we observe the structural differences listed in the theorem below. We will formalize these differences in the remainder of this subsection.
\begin{thm}\label{prop:four_point_geometry}
Denote by $\Gamma_1,...,\Gamma_n$ the flow matching cells and by $\mathcal L_1,...,\mathcal L_n$ the Laguerre cells \eqref{eq:cells_laguerre} corresponding to the atoms $a_1,...,a_n$. Then, in our example:
\begin{enumerate}[nosep,leftmargin=2em]
    \item[(i)] There exists $k$ such that the flow matching cell $\Gamma_k$ is unbounded, but the corresponding Laguerre cell $\mathcal L_k$ is bounded.
    \item[(ii)] There exist $k$ and $l$ such that the Laguerre cells $\mathcal L_k$ and $\mathcal L_l$ are neighbored (i.e., $\overline{\mathcal L_k}\cap\overline{\mathcal L_l}\neq \emptyset$), but $\Gamma_k$ and $\Gamma_l$ are not neighbored (i.e., $\overline{\Gamma_k}\cap\overline{\Gamma_l}=\emptyset$).
    \item[(iii)] There exists $k$ such that $\Gamma_k$ is non-convex.
    \item[(iv)] There exist $k$ and $l$ such that the boundary $\overline{\Gamma_k}\cap\overline{\Gamma_l}$ is not affine.
\end{enumerate}
\end{thm}

In order to formally show the theorem, we prove two lemmas, which determine certain subsets of the FM cells $\Gamma_k$. These subsets are illustrated in \Cref{fig:counterex}~(c). The proofs are deferred to \Cref{app:counterex_proofs}.
The first lemma shows that the flow map $\gamma$ maps all points on the red rays in \Cref{fig:counterex}~(c) to the central point $a_4=(0,0)$. 

\begin{lem}\label{lem1:regular_triangle_example}
In our example, it holds $\{(-c\cos(\frac{2k\pi}{3}),-c\sin(\frac{2k\pi}{3})):c\geq 0\}\subseteq \Gamma_4$ for $k=1,2,3$. 
\end{lem}

In particular, this implies that the red cell $\Gamma_4$ is unbounded for Flow Matching, whereas the corresponding Laguerre cell is bounded. It also follows that, unlike in the OT case, the cells $\Gamma_1$ (blue), $\Gamma_2$ (green), and $\Gamma_3$ (yellow) do not all share a common boundary.
To further identify the geometry of the cells, we next show that for each $k=1,2,3$ (blue green and yellow cells), the half-line $\{c\,a_k : c \geq 0\}$  eventually lies entirely inside $\Gamma_k$.

\begin{lem}\label{lem2:regular_triangle_example}
    In our example, there exists for $k=1,2,3$ some $c_0>0$ such that $\{ca_k:c\geq c_0\}\subseteq \Gamma_k$.
\end{lem}

Combining the two lemmas above, it follows that the red cell $\Gamma_4$ is not convex and that the boundaries between $\Gamma_4$ and $\Gamma_k$, for $k=1,2,3$, cannot be straight lines. Applying both lemmas we obtain the statements from \Cref{prop:four_point_geometry}.

\subsection{Non-Monotone Flow Maps}

An interesting consequence of the non-affine boundaries of the cells $\Gamma_k$ is that the flow map is not monotone.
Monotonicity of flow maps plays an important role in the context of the reflow algorithm \cite{liu2023flow}.
Reflow is an iterative procedure that repeatedly applies flow matching, using at each step the coupling $(\mu_0,\gamma_\#\mu_0)$ as the initial coupling, where $\gamma$ denotes the flow map learned by the previous flow matching iteration.
More precisely, \cite{bansal2025wassersteinconvergencestraightnessrectified} prove that reflow yields straight trajectories within the second iteration whenever the flow map in the first iteration is monotone.
A similar condition was derived in \cite{PSM2026}.

More specifically, \cite{bansal2025wassersteinconvergencestraightnessrectified} assume that the flow map $\gamma=\gamma_1$ from the flow matching algorithm with independent coupling is differentiable and fulfills
\begin{align} \label{eq:ass_bensal}
      \inf_x \lambda_{min}\left( J_{\gamma}(x) + J_{\gamma}(x)^\top \right) \geq 0,
\end{align}
or equivalently, that $\gamma$ is differentiable and monotone.
In the four-points example from the previous subsection, we can now observe that $\gamma$ is non-monotone by the next lemma. The proof is deferred to Appendix \ref{app:non_monotony}.
\begin{lem}\label{lem3:regular_triangle_example}
    In the semi-discrete example of Figure~\ref{fig:counterex}, there are $c >0$ and $x_1, x_2 \in \R^d$ such that
    $
     \langle \gamma(x_2)  - \gamma(x_1), x_2 - x_1 \rangle <  - c
     $.
\end{lem}
Of course our example is formulated in the semi-discrete regime, where $\mu_1$ is not absolutely continuous (such that our setting differs from the one in \cite{bansal2025wassersteinconvergencestraightnessrectified}). However, by applying the smoothing from \Cref{lem:GMM_0}, we can derive the same result for the Gaussian mixture case. As in \Cref{sec:simplifications} we denote by $\gamma^\epsilon$ the flow map of the flow matching algorithm in the Gaussian mixture case, where each mode has standard deviation $\epsilon$. Then, we obtain that the monotonicity assumption of \cite{bansal2025wassersteinconvergencestraightnessrectified} is violated. The proof can be found in Appendix \ref{app:non_monotony}.

\begin{lem}\label{lem3:GMM_1} There are $\varepsilon > 0$ and $c>0$ such that $
      \inf_x \lambda_{min}\left( J_{\gamma^\varepsilon_1}(x) + J_{\gamma^\varepsilon_1}(x)^\top \right) < - c$.
\end{lem}
More generally, one can see that in the semi-discrete setting (and thus also for gaussian mixture setting with sufficiently small standard deviations $\varepsilon$) the flow map $\gamma$ can only be monotone if the borders between the cells $\Gamma_k$ are affine. While we only have proved that for our specific examples, we observe numerically that this is only true in very specific cases. Indeed, in Figure~\ref{fig:monotony}, we test condition \eqref{eq:ass_bensal} (and more generally monotonicity) for our four-point example and observe that it exactly fails around the boundaries of the cells.

\begin{rem}
    Although several works (\citealp{bansal2025wassersteinconvergencestraightnessrectified,PSM2026,reu2026gradient}) employ monotone couplings to construct straight flows, the same authors consistently remark that monotonicity is not a necessary criterion. In fact, in our four-point example, the second reflow appears to be straight, since the regions are ``star-shaped'' with centers $a_k$: for every point in the cell $x \in \Gamma_k$, the whole segment $[x, \gamma(x) = a_k]$ remains entirely within the cell. This guarantees that the second reflow iteration has straight trajectories.
\end{rem}

\begin{figure}[h]
    \centering
    \begin{subfigure}{0.48\linewidth}
    \centering
    \includegraphics[width=0.75\linewidth]{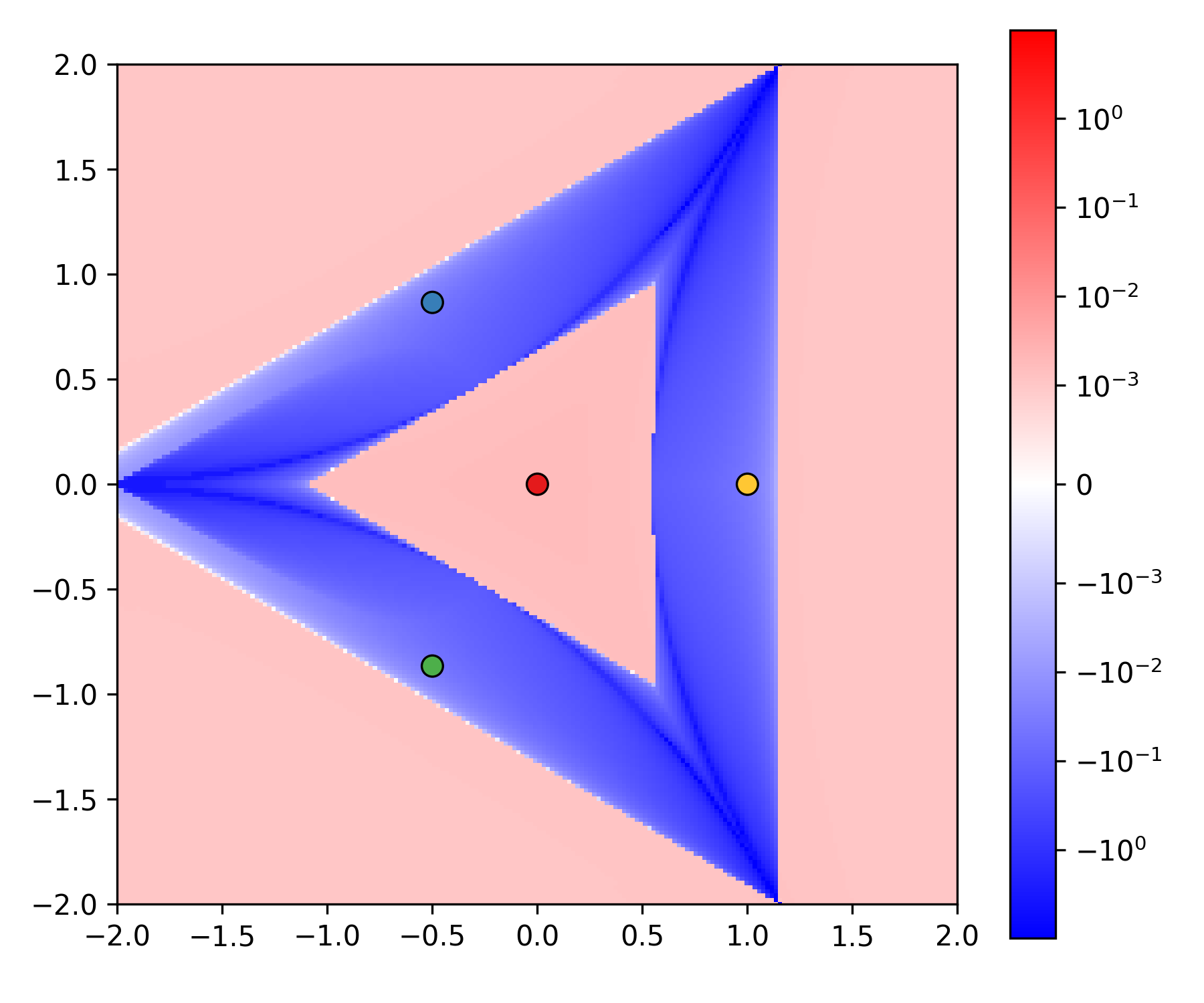}
     \caption{ $\min_y \langle \gamma_1(y)  - \gamma_1(x), y - x \rangle$}
     \end{subfigure}
    \begin{subfigure}{0.48\linewidth}
    \centering
    \includegraphics[width=0.75\linewidth]{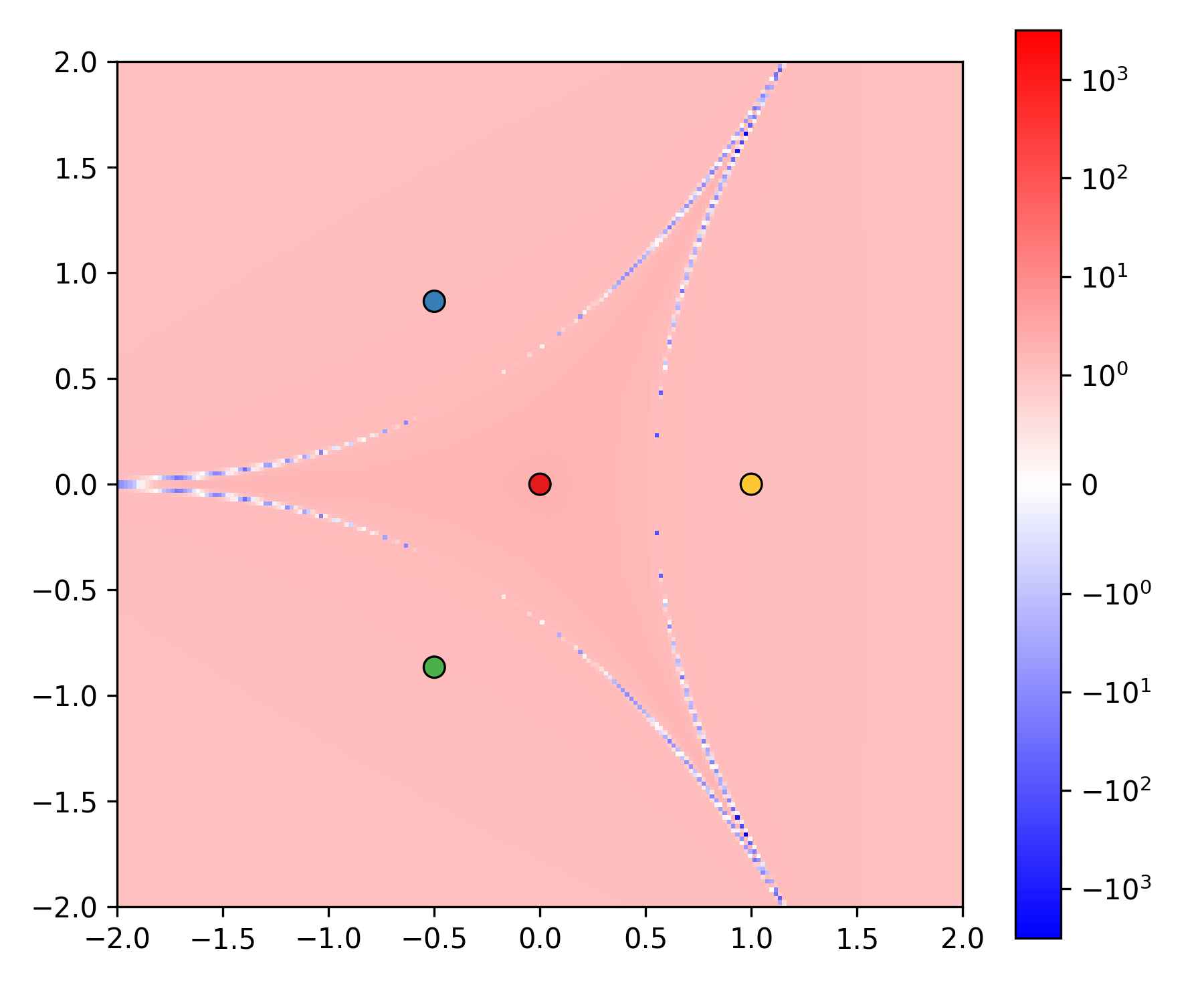}
     \caption{ $\lambda_{min}\left( J_{\gamma^\varepsilon_1}(x) + J_{\gamma^\varepsilon_1}(x)^\top \right)$}
     \end{subfigure}
     \caption{Illustration of the non-monoticity of the flow map, in the four-points configuration of Section~\ref{sec:counter}. 
    Figure (a) illustrates non-monoticity in the semi-discrete setting (Lemma~\ref{lem3:regular_triangle_example}), while Figure (b) illustrates Lemma~\ref{lem3:GMM_1} after Gaussian smoothing of the target atoms. In both cases, the loss of monotonicity is concentrated around the cell boundaries, suggesting that it is a geometric consequence of the non-affine FM boundaries.}
    \label{fig:monotony}
\end{figure}


\section{Conclusion and Discussion}

In this paper, we studied the geometry induced by the exact semi-discrete Flow Matching vector field, taking the viewpoint that this field is a canonical object in its own right and a useful theoretical proxy for the objectives used in generative modeling. Our main contribution is to show that the associated terminal assignment regions combine strong global topological structure with genuinely non-OT geometry. 
This gives a first qualitative picture of semi-discrete FM tessellations and clarifies in which sense FM can and cannot be viewed as a surrogate for semi-discrete OT.

\paragraph{Limitations and future directions.}
One limitation of our analysis is that it concerns the \emph{exact} semi-discrete FM vector field, whereas practical generative models only learn an approximation of it. Our results therefore do not directly characterize the geometry of the \emph{trained} velocity field or of the transport induced by a finite neural network after optimization. The exact cells should instead be viewed as a reference geometry: they identify which structural biases are already present in the objective before approximation, architecture, and optimization enter the picture. An important next step is to understand to what extent these qualitative properties survive training, and whether the learned velocity field inherits a meaningful analogue of the exact assignment cells in discrete or nearly discrete regimes. As a first numerical indication, Appendix~\ref{app:mnist} reports a MNIST experiment in which a simple neural FM model produces assignment cells that closely resemble the exact closed-form FM cells on a two-dimensional PCA slice of the latent space. 
Nevertheless, a systematic understanding of when and how these qualitative properties survive training remains open. 

Another limitation is that our results do not settle whether FM should be regarded as a useful proxy for OT from a geometric viewpoint. Our analysis shows that exact semi-discrete FM departs from Laguerre geometry outside effectively one-dimensional configurations. At the same time, this leaves open the possibility that trained or regularized FM models produce cells that are in some sense ``nicer'' than the exact ones, for instance smoother, more stable, or closer to OT cells after approximation effects are taken into account. Clarifying this point would help distinguish which geometric phenomena are intrinsic to FM and which are artifacts of the exact semi-discrete limit.

More broadly, several natural questions remain open. A first direction is to understand how the cells depend on the interpolation itself: how do they change for non-affine interpolations, or after adding noise to the target and considering smoothed variants such as Gaussian mixtures? A second direction is dynamical. One may study the geometry not only of the terminal cells $\Gamma_k$, but also of the intermediate cells defined along the flow, as well as their dependence on the locations of the support points $(a_k)_k$. A third direction concerns reflow and straightness: our counterexample rules out a general monotonicity principle in the present setting, but it remains unclear under which assumptions a second reflow becomes straight. 
Additionally, we conjecture that the assumption of \Cref{prop:contractible} is always fulfilled. While this was true in numerical simulations, it is important to find a formal proof.
Finally, it would be interesting to extend the present picture beyond the equal-weight distinct-atom setting considered here, both to general discrete weights and to asymptotic regimes such as a mean-field limit where the number of support points grows.

\section*{Acknowledgments}
This research was funded, in part, by the Agence nationale de la recherche (ANR), through the PEPR PDE-AI project (ANR-23-PEIA-0004). JH acknowledges funding by the Deutsche Forschungsgemeinschaft (DFG,
German Research Foundation) within project no 530824055. This work was performed using HPC resources from GENCI–IDRIS (Grant 2025-AD011015483R1).




\bibliographystyle{plainnat}
\bibliography{ref_neurips}

\appendix
\begin{appendices}

\section{Existence and Uniqueness of the flow map in the Semi-discrete case}
\label{sec:existence_uniqueness}

Several papers (including \cite{bansal2025wassersteinconvergencestraightnessrectified,mena2025statistical}) studied sufficient conditions for the existence and uniqueness of the flow map $\gamma_t$.
However, in this paper we only study the case where $\mu_0\sim\mathcal N(0,I)$ and where $(X_0,X_1)\sim\mu_0\otimes\mu_1$ is the independent coupling. In this case, the result is directly clear (for $\mu_0$-almost every $x$) without further assumptions. Due to the smoothness of the Gaussian, one finds that $v_t^*$ is globally Lipschitz on $[0,1-\epsilon]\times\R^d$ for any $\epsilon\in(0,1)$. In particular, we get that the flow map exists and is unique on $[0,1)\times\R^d$. Because of 
\begin{align*}
\left(\int \int_0^1 \|\dot \gamma_t(x)\| d t d\mu_0(x)\right)^2&\leq
 \int \int_0^1 \|\dot\gamma_t(x)\|^2 d t d\mu_0(x)
 =\int \int_0^1 \|v^*_t(\gamma_t(x))\|^2 d t d\mu_0(x)
 \\&=\int_0^1\int \|v^*_t(x)\|^2 d\mu_t(x)d t\leq \E{\|X_1-X_0\|^2}<\infty,
\end{align*}
we also obtain that the trajectories $t\mapsto\gamma_t(x)$ have finite length (a.e.) and can uniquely be extended to time $t=1$.

\section{Proofs of the Simplifications}

\subsection{Proof of \Cref{prop:only_depends_on_span}}
\label{proof:span_reduction}

By \Cref{cor:scaling} it suffices to consider the case $a_1=0$ in which case $E=\Span(a_k)_{1\leq k\leq n}$.
Then, for $1 \leq k \leq n$, $x \in \R^d$, $t \in [0,1)$, we have by the Pythagorean theorem that
$$
\begin{aligned}
\|x-ta_k\|^2 = \|x^\parallel-ta_k\|^2 +  \|x^\perp\|^2.
\end{aligned}
$$
Consequently,
$$
\begin{aligned}
    \alpha_k(t,x)
    & =\displaystyle\frac{e^{-\frac{1}{2(1-t)^2}(\|x^\parallel-ta_k\|^2 +  \|x^\perp\|^2)}}{\displaystyle\sum_{i=1}^ne^{-\frac{1}{2(1-t)^2}(\|x^\parallel-ta_k\|^2 +  \|x^\perp\|^2)} } = \displaystyle\frac{e^{-\frac{1}{2(1-t)^2}\|x^\parallel-ta_k\|^2 }}{\displaystyle\sum_{i=1}^ne^{-\frac{1}{2(1-t)^2}\|x^\parallel-ta_k\|^2 } } = \alpha_k(t,x^\parallel).
\end{aligned}
$$
Finally,
$$
\begin{aligned}
    \frac{dx_t}{dt} 
    = \frac{dx^\parallel_t}{dt} + \frac{dx^\perp_t}{dt} 
    = \underbrace{\sum_{k=1}^n \alpha_k(t,x_t^\parallel) \frac{a_k-x_t^\parallel}{1-t}}_{\in \Span(a_k)_{1 \leq k \leq n}} + \underbrace{-\frac{x_t^\perp}{1-t}}_{\in \Span(a_k)_{1 \leq k \leq n}^\perp}
\end{aligned}
$$
This completes the proof.\hfill$\Box$





\subsection{Proof of \Cref{lem:GMM_0}}
\label{app:GMM_0}
We start with the following intermediate results which uniformly controls the distance between $\gamma^\varepsilon_t$ and $\gamma_t$ for $t \leq T < 1$.
\begin{lem}  \label{lem:convergence_eps}
   For every $R >0$ and $T < 1$, 
   \begin{align}
       \sup_{\norm{x} \leq R} \sup_{0 \leq t \leq T} \| \gamma^\varepsilon_t(x) -  \gamma_t(x)\| \to 0
   \end{align}
\end{lem}
\begin{proof}
Let $e_t^\varepsilon(x) = \gamma^\varepsilon_t(x)  - \gamma_t(x)$. Using $\dot e_t^\varepsilon(x)  = v_t^\varepsilon(\gamma_s^\varepsilon(x))-v_t(\gamma_s(x))$ and $e^\varepsilon_0(x) = 0$, we have:
\begin{align}
\norm{e_t(x)}
&\le \int_0^t
\norm{v_s^\varepsilon(\gamma_s^\varepsilon(x))-v_s(\gamma_s(x))} ds \\
&\le \int_0^t
\norm{v_s^\varepsilon(\gamma_s^\varepsilon(x))-v_s^\varepsilon(\gamma_s(x))}ds
+\int_0^t
\norm{v_s^\varepsilon(\gamma_s(x))-v_s(\gamma_s(x))}ds  \label{eq:bound_sup}
\end{align}
We now prove that for bounded initialization $x$ and $t \leq T$, the iterates $\gamma_t(x)$ and $\gamma^\varepsilon_t(x)$ remain uniformly bounded. Indeed, we can bound  
\begin{equation}
    \norm{v_t^\varepsilon(x)} \leq \sup_{a_k} \norm{\frac{(1-t)(a_k-x) + t \varepsilon^2 x}{(1-t)^2 + t^2 \varepsilon^2}} \leq C_T (1 + \norm{x}) 
\end{equation}
where $C_T$ depends on $T$, not on $\epsilon$.
In the same way $\norm{v_t(x)} \leq C_T (1 + \norm{x})$ and thus integrating, we get that there is $M >0$ such that
\begin{equation}
\sup_{\norm{x} \leq R} \sup_{0 \leq t \leq T} \| \gamma^\varepsilon_t(x)\| \leq M, \quad \sup_{\norm{x} \leq R} \sup_{0 \leq t \leq T} \| \gamma_t(x)\| \leq M 
\end{equation}
In order to handle the first term of~\eqref{eq:bound_sup}, note that $v_t^\varepsilon$ is uniformly Lipschitz on $K_T \eqdef [0,T] \times B(0,M)$: there is $L_T >0$, such that for all $t \in [0,T]$ and $x,y \in B(0,R)$, 
\begin{equation}
\norm{v_s^\varepsilon(\gamma_s^\varepsilon(x))-v_s^\varepsilon(\gamma_s(x))} \leq L_T \, \norm{e_s(x)} . 
\end{equation}
Therefore, we get 
\begin{align}
\norm{e_t(x)}
&\le L_T \int_0^t \norm{e_s(x)} ds
+ t \sup_{0 \leq t \leq T, \norm{x} \leq M} \norm{v_s^\varepsilon(x)-v_s(y)}ds .
\end{align}
And applying Gronwall lemma: 
\begin{align}
\sup_{t \leq T, \norm{x} \leq R}\norm{e_t(x)}
&\le Te^{T L_T}\sup_{0 \leq t \leq T, \norm{x} \leq M} \norm{v_s^\varepsilon(x)-v_s(y)}ds .
\end{align}
Finally, using that $v_t^\varepsilon$ converges uniformly on $K_T$ to its $\varepsilon=0$ counterpart, the right-hand side tends to $0$.
\end{proof}

Now, we can prove \Cref{lem:GMM_0}. 
\begin{proof}[Proof of \Cref{lem:GMM_0}]
Let \(k\in\{1,\dots,n\}\) be such that $\gamma_t(x)\to a_k$. Let
\[
d_k:=\min_{i\neq k}\norm{a_i-a_k},
\]
and choose \(0<\rho<d_k/4\). There exists \(T_0<1\) such that
\[
\gamma_t(x)\in B(a_k,\rho/8)
\qquad\text{for all }t\in[T_0,1).
\]
By \Cref{lem:convergence_eps}, for this fixed \(T_0\), there exists \(\varepsilon_0>0\) such that
\[
\norm{\gamma_{T_0}^\varepsilon(x)-\gamma_{T_0}(x)}\le \rho/8
\qquad\text{for all }0<\varepsilon<\varepsilon_0.
\]
Hence
\begin{equation}\label{eq:init_small}
\norm{\gamma_{T_0}^\varepsilon(x)-a_k}
\le \rho/4
\qquad\text{for all }0<\varepsilon<\varepsilon_0.
\end{equation}
We now define $\tau_\varepsilon$ as the time where the trajectory $\gamma_t^\varepsilon(x)$ leaves $B(a_k,\rho)$:
\begin{equation}
\tau_\varepsilon:=\inf\bigl\{t\in[T_0,1]:\ \gamma_t^\varepsilon(x)\notin B(a_k,\rho)\bigr\},
\end{equation}
with the convention \(\tau_\varepsilon=1\) if the set is empty. Then
\begin{equation}
\gamma_t^\varepsilon(x)\in B(a_k,\rho)
\qquad\text{for all }t\in[T_0,\tau_\varepsilon].
\end{equation}
Set
\begin{align}
D_\varepsilon(t):=(1-t)^2+t^2\varepsilon^2,
\qquad
r_t^\varepsilon(x):=\gamma^\varepsilon_t(x)-a_k,  \qquad m_t^\varepsilon(z):=\sum_{i=1}^n \alpha_i^\varepsilon(t,z)\,a_i.
\end{align}
The flow map $\gamma_t^\varepsilon(x)$ verifies the ODE
\begin{align}
\frac{d}{dt} \gamma_t^\varepsilon(x)
=
\frac{(1-t)\big(m_t^\varepsilon(\gamma_t^\varepsilon(x))-\gamma_t^\varepsilon(x)\big)+t\varepsilon^2 \gamma_t^\varepsilon(x)}{D_\varepsilon(t)},
\end{align}
which re-writes for $r_t^\varepsilon(x)$ as
\begin{equation}\label{eq:r_eps}
\frac{d}{dt} r_t^\varepsilon(x)
=
-\frac{1-t-t\varepsilon^2}{D_\varepsilon(t)}\,r_t^\varepsilon(x)
+
\frac{1-t}{D_\varepsilon(t)}\bigl(m_t^\varepsilon(\gamma_t^\varepsilon(x))-a_k\bigr)
+
\frac{t\varepsilon^2}{D_\varepsilon(t)}\,a_k.
\end{equation}
We now introduce the rescaled 
\begin{align}
q_t^\varepsilon(x):=\frac{r_t^\varepsilon(x)}{\sqrt{D_\varepsilon(t)}}.
\end{align}
in order to factor out the scale \(\sqrt{D_\varepsilon(t)}\) of \(r_t^\varepsilon\) near \(t=1\). It verifies the ODE:
\begin{align}\label{eq:q_eps}
\frac{d}{dt}q_t^\varepsilon(x)
&= 
F_\varepsilon(t),
\end{align}
where
\[
F_\varepsilon(t):=
\frac{1-t}{D_\varepsilon(t)^{3/2}}
\bigl(m_t^\varepsilon(\gamma_t^\varepsilon(x))-a_k\bigr)
+
\frac{t\varepsilon^2}{D_\varepsilon(t)^{3/2}}\,a_k.
\]

In order to study  $F_\epsilon$, we next estimate the norm $\norm{m_t^\varepsilon(z) - a_k}$ for any $z \in B(a_k,\rho)$. We have
\begin{equation}
\norm{m_t^\varepsilon(z)-a_k}
=
\norm{\sum_{i\neq k}\alpha_i^\varepsilon(t,z)(a_i-a_k)} \leq \sum_{i\neq k} \frac{\alpha_i^\varepsilon(t,z)}{\alpha_k^\varepsilon(t,z)}\norm{a_i-a_k}
\end{equation}

For $ i\neq k$, define
\begin{align}
\Delta_i(t,z) &:=\norm{z-ta_i}^2-\norm{z-ta_k}^2 \\
&= t\bigl(\norm{z-a_i}^2-\norm{z-a_k}^2\bigr)
-t(1-t)\big(\norm{a_i}^2-\norm{a_k}^2\big),
\end{align}
We have that $z\in B(a_k,\rho)$ implies $
\norm{z-a_i}^2-\norm{z-a_k}^2\ge c_0$
for some $c_0>0$. Thus, after increasing \(T_0\) if necessary, for all $t\in[T_0,1]$ and $z\in B(a_k,\rho)$, for any $i\neq k$, $\Delta_i(t,z)\ge c_0/2$
and
\begin{align}
\frac{\alpha_i^\varepsilon(t,z)}{\alpha_k^\varepsilon(t,z)}
=
\exp\!\left(
-\frac{\Delta_i(t,z)}{2((1-t)^2+t^2\varepsilon^2)}
\right)
\le
\exp\!\left(
-\frac{c_0}{4((1-t)^2+t^2\varepsilon^2)}
\right).
\end{align}
Summing over \(i\neq k\), we obtain constants \(c,C>0\)
such that 
\begin{equation}\label{eq:m_small_general}
\norm{m_t^\varepsilon(z)-a_k}
\le
C\exp\!\left(-\frac{c}{(1-t)^2+\varepsilon^2}\right)
\qquad
\text{for all }t\in[T_0,1],\ z\in B(a_k,\rho),\ 0<\varepsilon<\varepsilon_0.
\end{equation}
Hence, on \([T_0,\tau_\varepsilon]\), estimate~\eqref{eq:m_small_general} applies: the first term of $F_\varepsilon(t)$
is then bounded by
\begin{equation}
C\frac{1-t}{((1-t)^2+\varepsilon^2)^{3/2}}
\exp\!\left(-\frac{c}{(1-t)^2+\varepsilon^2}\right),
\end{equation}
whose integral on \([T_0,1]\) is uniformly bounded in $\varepsilon$. 
Note also that the last term of $F_\epsilon$ is uniformly bounded in $L^1([T_0, 1])$. Indeed:
\begin{equation}
I_{T_0} \eqdef \int_{T_0}^1 \frac{t\varepsilon^2}{D_\varepsilon(t)^{3/2}} dt\leq \int_{T_0}^1 \frac{\varepsilon^2}{((1-t)^2 + T_0^2 \varepsilon^2 )^{3/2}} dt = \frac{1}{T_0^2} \int_{T_0}^1 \frac{1}{(\frac{(1-t)^2}{T_0^2 \varepsilon^2} + 1)^{3/2}} \frac{1}{T_0 \varepsilon}dt
\end{equation}
and a change of variable $u = \frac{1-t}{T_0 \varepsilon}$ gives $I_{T_0} \leq  \frac{1}{T_0^2}$.

Thus, integrating~\eqref{eq:q_eps} on \([T_0,\tau_\varepsilon]\), we obtain
\begin{equation}
\norm{q_t^\varepsilon(x)}
\le
\norm{q_{T_0}^\varepsilon(x)}+C_1
\qquad\text{for all }t\in[T_0,\tau_\varepsilon].
\end{equation}
Multiplying by \(\sqrt{D_\varepsilon(t)}\), this gives
\begin{equation}
\norm{r_t^\varepsilon(x)}
\frac{\sqrt{D_\varepsilon(t)}}{\sqrt{D_\varepsilon(T_0)}}\norm{r_{T_0}^\varepsilon(x)}
+
C_1\sqrt{D_\varepsilon(t)} .
\end{equation}
Since $D_\varepsilon = (1-t)^2 + t^2 \varepsilon^2$ is convex in $t$, its maximum on $[T_0,1]$ is attained at one of the endpoints. Now, for $\varepsilon < \varepsilon_0 < \frac{1-T_0}{1+T_0}$, we get $D_\varepsilon(1) = \varepsilon^2  \leq D_\varepsilon(T_0) $, for all $t \in [T_0,1]$, $D_\varepsilon(1) \leq D_\varepsilon(T_0)$. Thus:
\begin{equation}
\norm{r_t^\varepsilon(x)}
\le\norm{r_{T_0}^\varepsilon(x)}
+
C_1\sqrt{D_\varepsilon(t)}
\qquad\text{for all }t\in[T_0,\tau_\varepsilon].
\end{equation}
Using \eqref{eq:init_small} and \(\sqrt{D_\varepsilon(t)}\le (1-t)+\varepsilon\), we get
\begin{equation}\label{eq:bootstrap_bound_general}
\norm{\gamma_t^\varepsilon(x)-a_k}
\le
\frac{\rho}{4}
+
C_1\bigl((1-t)+\varepsilon\bigr)
\qquad\text{for all }t\in[T_0,\tau_\varepsilon].
\end{equation}

Now choose \(T_0\) closer to \(1\), and decrease \(\varepsilon_0\) if necessary, so that
\begin{equation}
\frac{\rho}{4}
+
C_1\bigl((1-T_0)+\varepsilon_0\bigr)
<\rho.
\end{equation}
Then \eqref{eq:bootstrap_bound_general} implies
\begin{equation}
\gamma_t^\varepsilon(x)\in B(a_k,\rho)
\qquad\text{for all }t\in[T_0,\tau_\varepsilon].
\end{equation}
By the definition of \(\tau_\varepsilon\), this forces \(\tau_\varepsilon=1\).

We have therefore proved that there is $T_0 >0$ and $\varepsilon_0 >0$ such that for all \(t\in[T_0,1]\) and all \(0<\varepsilon<\varepsilon_0\),
\begin{equation}
\norm{\gamma_t^\varepsilon(x)-a_k}
\le
\rho
\end{equation}
\end{proof}
\begin{cor}\label{cor:cell_convergence}
Let
$
\Gamma^\varepsilon_k
=
\left\{
x\in\mathbb{R}^d:
\|\gamma^\varepsilon(x)-a_k\|
<
\|\gamma^\varepsilon(x)-a_l\|
\text{ for all } l\neq k
\right\}.
$
Then, for every $x\in \bigcup_{j=1}^n \Gamma_j$ and every $k\in\{1,\dots,n\}$,
$$
    \mathbf{1}_{\Gamma_k^\varepsilon}(x)
    \longrightarrow
    \mathbf{1}_{\Gamma_k}(x)
    \qquad\text{as }\varepsilon\to0.
$$
\end{cor}
\begin{proof}
Let $x\in \bigcup_{j=1}^n \Gamma_j$. Then there exists a unique index $k\in\{1,\dots,n\}$ such that $x\in\Gamma_k$. By definition of $\Gamma_k$, this means that
$
    \gamma_1(x)=a_k.
$
By Proposition~\ref{lem:GMM_0}, we have
$
    \gamma^\varepsilon(x)\to a_k
$
as $\varepsilon\to0$.

Since the atoms $a_1,\dots,a_n$ are pairwise distinct, set
$
    r_k := \frac12 \min_{j\neq k}\|a_k-a_j\|>0.
$

Using Proposition~\ref{lem:GMM_0}, for sufficiently small $\varepsilon>0$, we have
$$
    \|\gamma^\varepsilon(x)-a_k\|<r_k
$$
and for every $j\neq k$,
$$
    \|\gamma^\varepsilon(x)-a_j\|
    \geq
    \|a_k-a_j\|-\|\gamma^\varepsilon(x)-a_k\|
    >
    2r_k-r_k
    =
    r_k.
$$
Therefore,
$
\|\gamma^\varepsilon(x)-a_k\|
    <
\|\gamma^\varepsilon(x)-a_j\|
$
for all $j\neq k$ and thus $x\in\Gamma_k^\varepsilon$. In particular, $x\notin\Gamma_j^\varepsilon$ for all $j\neq k$.
Consequently, for all sufficiently small $\varepsilon>0$,
$
    \mathbf{1}_{\Gamma_k^\varepsilon}(x)=1=\mathbf{1}_{\Gamma_k}(x),
$
while for every $j\neq k$,
$
    \mathbf{1}_{\Gamma_j^\varepsilon}(x)=0=\mathbf{1}_{\Gamma_j}(x).
$
\end{proof}

\section{Proofs for Topology of the FM Cells}
\label{appendix:proofs_topology}

\subsection{Proof of \Cref{prop:into_the_ball_implies_convergence_to_ak}}\label{proof:into_the_ball_implies_convergence_to_ak}

We start with proving stating that $\alpha_k$ converges to one in a ball around $a_k$ for late enough time. More precisely, we have the following lemma.
\begin{lem}\label{lem:alpha_abschaetzung}
Let $a_k\neq a_l$ for $k\neq l \in\{1,...,N\}$ and define $r<\frac12 \min_{k\neq l}\|a_k-a_l\|$. Then, there exists a time $t_0\in(0,1)$ and $\delta>0$ such that for all $t>t_0$ it holds
$$
\alpha_k(t,x)>\frac{\exp\left(\frac{\delta}{2(1-t)^2}\right)}{\exp\left(\frac{\delta}{2(1-t)^2}\right)+n-1}\to 1\quad\text{as }t\to 1.
$$
\end{lem}
\begin{proof}
There exists $\delta>0$ such that for all $x\in\bar B_r(a_k)$ and $l\neq k$  it holds
$$
\|x-a_k\|^2<\|x-a_l\|^2-2\delta.
$$
Thus there exists $t_0\in(0,1)$ such that for all $t>t_0$it holds
$$
\|x-ta_k\|^2<\|x-ta_l\|^2-\delta.
$$
Furthermore, we have
$$
\tilde\alpha_{k,l}(t,x)\coloneqq \frac{\exp\left(-\frac{\|x-ta_k\|^2}{2(1-t)^2}\right)}{\exp\left(-\frac{\|x-ta_l\|^2}{2(1-t)^2}\right)}>\exp\left(\frac{\delta}{2(1-t)^2}\right).
$$
This implies (for such $t,x$)
$$
\alpha_k(t,x)=\frac{1}{1+\sum_{k\neq l}\frac{1}{\tilde\alpha_{k,l}(t,x)}}>\frac{1}{1+(n-1)\exp\left(-\frac{\delta}{2(1-t)^2}\right)}=\frac{\exp\left(\frac{\delta}{2(1-t)^2}\right)}{\exp\left(\frac{\delta}{2(1-t)^2}\right)+n-1} $$
\end{proof}

This allows us now to prove that the velocity $v_t(x)$ on $\partial B_r(a_k)$ points inwards the ball for large enough time. We formalize
this observation in the following lemma.
\begin{lem}\label{lem:velocity_points_inwards}
Let $a_k\neq a_l$ for $k\neq l \in\{1,...,N\}$ and define $r<\frac12 \min_{k\neq l}\|a_k-a_l\|$. Then there exists $\underline{t}_r\in(0,1)$ such that for all $t>\underline{t}_r$ and $x\in\partial B_r(a_k)$ we have that $\langle a_k-x,v_t(x)\rangle>0$.
\end{lem}
\begin{proof}
We have for $M\coloneqq \max_{l}\|a_k-a_l\|$
\begin{align*}
(1-t)\langle a_k-x,v_t(x)\rangle &= \alpha_k(t,x)\|a_k-x\|^2 + \sum_{l\neq k} \alpha_l(t,x)\langle a_k-x,a_l-x\rangle\\
&\geq \alpha_k(t,x) r^2 - \sum_{l\neq k}\alpha_l(t,x)\|a_k-x\|\|a_l-x\|\\
&\geq \alpha_k(t,x) r^2 - r(M+r)\sum_{l\neq k}\alpha_l(t,x)\\
&=\alpha_k(t,x) r^2 - r(M+r)(1-\alpha_k(t,x))
\end{align*}
which converges to $r^2$ for $\alpha_k(t,x)\to 1$.
Since $\alpha_k(t,x)$ converges to $1$ uniformly on $\bar B_r(a_k)$ by \Cref{lem:alpha_abschaetzung}, we obtain that there exists some time $\underline{t}_r$ such that the above expression is positive for all $t>\underline{t}_r$.
\end{proof}

This lemma directly implies that a trajectory will never leave the ball $\bar B_r(a_k)$ once it entered it at late enough time as formalized in the following corollary.
\begin{cor}
\label{cor:stable_balls}
	Let $x \in \R^d$, if there exists $\underline{t}_r <t_0 < 1$ such that $\gamma_{t_0}(x) \in  \bar B_r(a_k)$ then $\forall t\geq t_0, \gamma_t(x) \in \bar B_r(a_k)$.
\end{cor}

Finally, we can prove the desired proposition, by putting convergence rates onto the ``inward pointing lemma'' \Cref{lem:velocity_points_inwards}.

\begin{proof}[Proof of \Cref{prop:into_the_ball_implies_convergence_to_ak}]
By contradiction, we suppose that $\gamma_{t}(x) \in  \bar B_r(a_k)$ for $\underline{t}_r < t < 1$ and $\gamma_t(x)$ does not converge to $a_k$ when $t \to 1$.
With this assumption, there exists $0 <\varepsilon_0 < r$ such that for all $\underline{t}_{\varepsilon_0} < T<1 $ there exists $t > T$ such that $\|\gamma_{t}(x)-a_k\| > 2\varepsilon_0$. 
It implies that there exists a sequence $(t_n)$ such that $t_n \to 1$, $t_n > \underline{t}_{\varepsilon_0}$ and $\|\gamma_{t_n}(x)-a_k\| > 2\varepsilon_0$. 
By \Cref{cor:stable_balls}, it implies that $\forall t \geq t_1$, $\|\gamma_{t}(x)-a_k\| > \varepsilon_0$.
Consequently, for all $t \geq t_1$ and with $M\coloneqq \max_{l}\|a_k-a_l\|$,
$$
\begin{aligned}
&\quad (1-t)\langle a_k-\gamma_t(x),v_t(\gamma_t(x))\rangle \\
&= \alpha_k(t,\gamma_t(x))\|a_k-\gamma_t(x)\|^2 + \sum_{l\neq k} \alpha_l(t,\gamma_t(x))\langle a_k-\gamma_t(x),a_l-\gamma_t(x)\rangle\\
&\geq \alpha_k(t,\gamma_t(x)) \varepsilon_0^2 - \sum_{l\neq k}\alpha_l(t,\gamma_t(x))\|a_k-\gamma_t(x)\|\|a_l-\gamma_t(x)\|\\
&\geq \alpha_k(t,\gamma_t(x)) \varepsilon_0^2  - r(M+r)\sum_{l\neq k}\alpha_l(t,\gamma_t(x))\\
&=\alpha_k(t,\gamma_t(x)) \varepsilon_0^2  - r(M+r)(1-\alpha_k(t,\gamma_t(x))) \\
& = \alpha_k(t,\gamma_t(x))(\varepsilon_0^2 +r^2+rM)- r(M+r) \\
& > \frac{\exp\left(\frac{\delta}{2(1-t)^2}\right)}{\exp\left(\frac{\delta}{2(1-t)^2}\right)+n-1}(\varepsilon_0^2 +r^2+rM)- r(M+r) \\
& = \frac{\exp\left(\frac{\delta}{2(1-t)^2}\right)(\varepsilon_0^2 +r^2+rM- r(M+r))- r(M+r)(n-1)}{\exp\left(\frac{\delta}{2(1-t)^2}\right)+n-1} \\
& = \frac{\exp\left(\frac{\delta}{2(1-t)^2}\right)\varepsilon_0^2 - r(M+r)(n-1)}{\exp\left(\frac{\delta}{2(1-t)^2}\right)+n-1} \\
& = \varepsilon_0^2-\frac{  (\varepsilon_0^2+r(M+r))(n-1)}{\exp\left(\frac{\delta}{2(1-t)^2}\right)+n-1} \\
\end{aligned}
$$

Let us consider $F(t) = \frac12\|\gamma_t(x)-a_k\|^2$, for all $t \geq t_1$,

$$
\begin{aligned}
	\dot F(t)
    &= \langle \gamma_t(x)-a_k,\dot\gamma_t(x)\rangle
	= -\langle a_k-\gamma_t(x),v_t(\gamma_t(x))\rangle\\
	&\leq \frac{1}{1-t}\frac{(\varepsilon_0^2+r(M+r))(n-1)}{\exp\left(\frac{\delta}{2(1-t)^2}\right)+n-1}-\frac{\varepsilon_0^2}{1-t}
\end{aligned}
$$

And, by integrating between $t_1$ and $t \geq t_1$,
$$
\begin{aligned}
	F(t)
	& \leq F(t_1) + \int_{t_1}^t\frac{1}{1-s}\frac{(\varepsilon_0^2+r(M+r))(n-1)}{\exp\left(\frac{\delta}{2(1-s)^2}\right)+n-1}ds- \int_{t_1}^t\frac{\varepsilon_0^2}{1-s}ds \\
	& = F(t_1) + (\varepsilon_0^2+r(M+r))(n-1)\underbrace{\int_{t_1}^t\frac{1}{1-s}\frac{1}{\exp\left(\frac{\delta}{2(1-s)^2}\right)+n-1}ds}_{<+\infty}+\varepsilon_0^2[\log(1-s)]_{t_1}^t \\
    & = C_{t_1,r,\varepsilon_0}+\varepsilon_0^2\log(1-t) \to -\infty \quad\text{as }t\to 1.
\end{aligned}
$$
Since we know that $F(t) \geq 0$, this is a contradiction and we are done.
\end{proof}

\newpage
\subsection{Proofs of the Topological Consequences}
\label{sec:proofs_topological}

\begin{prop}
\label{cor:Gamma_k_are_open_sets}
    For $1 \leq k \leq n$, $\Gamma_k$ is an open set.
\end{prop}
\begin{proof}
    For $1 \leq k \leq n$, let us show that $\Gamma_k = \bigcup_{\underline{t}_r < t < 1} \gamma_t^{-1}(B_r(a_k))$ which is an union of open set because $\gamma_t^{-1}$ is well-defined and continuous for $t<1$.

    First, $\Gamma_k \subset \bigcup_{\underline{t}_r < t < 1} \gamma_t^{-1}(B_r(a_k))$.
    Indeed, if $\lim_{t \rightarrow 1^{-}}\gamma_t(x) = a_k$, there exists a time $t_0$ such that $\forall t \geq t_0, \gamma_t(x) \in B_r(a_k)$.

    Second, by \Cref{prop:into_the_ball_implies_convergence_to_ak} and \Cref{cor:stable_balls}, $\bigcup_{t < \underline{t}_r < 1} \gamma_t^{-1}(B_r(a_k)) \subset \Gamma_k$.
    Indeed, if we suppose that $x \in  \gamma_{t_0}^{-1}(B_r(a_k))$ for $\underline{t}_r < t_0 <  1$ then by \Cref{prop:into_the_ball_implies_convergence_to_ak}, $x \in \Gamma_k$.
\end{proof}

\begin{prop}
\label{prop:Gamma_k_connected}
 $\Gamma_k$ is path-connected.
\end{prop}
\begin{proof}
Let $x_0,x_1\in \Gamma_k$ and take $r$ and $\underline{t}_r$ from \Cref{lem:velocity_points_inwards}. Since the mappings $t\mapsto \gamma_t(x_0)$ and $t\mapsto \gamma_t(x_1)$ are continuous, there exist time points $t_0, {t_1\in(\overline{t}},1)$ such that $\gamma_{{t_0}}(x_0),\gamma_{{t_1}}(x_1)\in \bar B_r(a_k)$.
Given \Cref{cor:stable_balls}, it implies that for all $t \geq t_0, \gamma_{t}(x_0) \in \bar{B}(a_k,r)$ and for all $t \geq t_1, \gamma_{t}(x_1) \in \bar{B}(a_k,r)$.
Without loss of generality we can assume that {$t_0>t_1$}. 
Now, we define $x_c=\gamma_{t_0}^{-1}((1-c)\gamma_{t_0}(x_0)+c\gamma_{t_0}(x_1))$ and note that the path $c\mapsto x_c$ is continuous and connects $x_0$ and $x_1$.
It remains to show that that $x_c\in\Gamma_k$. This is provided by \Cref{prop:into_the_ball_implies_convergence_to_ak}. Let $c \in [0,1]$, $\gamma_{t_0}(x_c) = (1-c)\gamma_{t_0}(x_0)+c\gamma_{t_0}(x_1) \in \bar B_r(a_k)$ (by convexity of the ball). Because $t_0 \geq \underline{t}_r$, by \Cref{prop:into_the_ball_implies_convergence_to_ak}, $x_c \in \Gamma_k$.
\end{proof}

To show that $\Gamma_k$ is actually simply connected, we first prove that for any compact subset $A\subseteq\Gamma_k$ we have that $\gamma_t(A)\subseteq \bar B_r(a_k)$ for $r$ small enough and $t$ large enough. More precisely, we have the following lemma.
\begin{lem}\label{lem:compact_sets_in_Br}
Let $A\subset \Gamma_k$ be compact and $r< \frac{1}{2}\min_{l\neq k}\|a_k-a_l\|$. Then, there exists some time $t^*\in(0,1)$ such that for all $t>t^*$ it holds $\gamma_t(A)\subseteq \bar B_r(a_k)$
\end{lem}
\begin{proof}
    Let $\underline{t}_r$ be from \Cref{prop:into_the_ball_implies_convergence_to_ak}. Then,
    we denote $S_k(t) \eqdef \gamma_t^{-1}(B_r(a_k))$ and observe that $\Gamma_k = \cup_{t\geq t_r} S_k(t)$ with increasing $S_k(t)$.
    Since $A$ is compact and the $S_k(t)$ for $t\geq t_r$ are open and covering $A$, there exists a finite subset $S_k(t_1), ..., S_k(t_T)$ such that $A\subset \cup_{i=1}^T S_k(t_T)=S_k(t^*)$ for $t^*=\max(t_1,...,t_T)$.
    In other words, $\gamma_{t^*}(A)\subseteq \bar B_r(a_k)$ which implies by \Cref{prop:into_the_ball_implies_convergence_to_ak,cor:stable_balls} also that $\gamma_{t}(A)\subseteq \bar B_r(a_k)$ for all $t\geq t^*$.
\end{proof}

Applying this lemma to an arbitrary loop yields the following proposition.

\begin{prop}
\label{prop:no_holes}
    $\Gamma_k$ is simply connected.
\end{prop}
\begin{proof}
    Take a continuous loop $f : S^1 \to \Gamma_k$. We have to build a homotopy $H : S^1 \times [0,1] \to \Gamma_k$ s.t. $H(.,0) = f$ and $H(., 1)$ is a constant loop.
    By \Cref{lem:compact_sets_in_Br}, we know that there exists some $t^*$ such that $f(S^1)\subseteq  \gamma_{t^*}^{-1}(\bar B_r(a_k))\subseteq \Gamma_k$.
    Since $\gamma_{t^*}^{-1}(\bar B_r(a_k))$ is diffeomorphic to a ball and thus simply connected and there exists a continuous $H_{t^*} : S^1 \times [0,1] \to \gamma_{t^*}^{-1}(\bar B_r(a_k))$ such that $H_{t^*}(.,0) = f$ and $H_{t^*}(., 1)$ is a constant loop.
\end{proof}

\subsection{Existence of the limit of $\gamma_t^{-1}(a_k)$ for $t\to\infty$ and Contractibility}\label{appendix:zentren}

First, we bound the gradient norm of $v_t$ globally and locally around $a_k$. 

\begin{lem}\label{lem:gradient_von_v}
There exists some $c>0$ such that $\|\nabla v_t(x)\|\leq \frac{c}{(1-t)^3}$.
Further, for any $\epsilon>0$ and $r< \frac12\min_{l\neq k}\|a_k-a_l\|$ there exists $\underline{t}\in(0,1)$ such that for any $t>\underline{t}$ we have that for all $x\in B_r(a_k)$ it holds $\|\nabla v_t(x)\|<\frac{1+\epsilon}{1-t}$.
\end{lem}
\begin{proof}
Recall that
$v_t(x)=-\frac{x}{1-t}+\frac{1}{1-t}\sum_{l=1}^n s_l(-\frac{\|x-ta_j\|^2}{2(1-t)^2})a_l$, where $s_l$ is the softmax function. Using the chain rule and the derivative $\partial_j s_l(x)=s_l(x)(\delta_{jl}-s_j(x))$ this implies that
\begin{align*}
\nabla v_t(x)&=-\frac{I}{1-t}-\frac{1}{(1-t)^3}\sum_{l=1}^n\alpha_l(t,x)\Big(a_l-\sum_{j=1}^n\alpha_j(t,x)a_j\Big)(x-ta_l)^T
\end{align*}
Noting that 
$$
\sum_{l=1}^n\alpha_l(t,x)\Big(a_l-\sum_{j=1}^n\alpha_j(t,x)a_j\Big)x^T=\sum_{l=1}^n\alpha_l(t,x)a_lx^T - \underbrace{\Big(\sum_{l=1}^n\alpha_l(t,x)\Big)}_{=1}\sum_{j=1}^n\alpha_j(t,x)a_jx^T=0,
$$
this simplifies to
$$
\nabla v_t(x)=-\frac{I}{1-t}+\frac{t}{(1-t)^3}\underbrace{\sum_{l=1}^n\alpha_l(t,x)\Big(a_l-\sum_{j=1}^n\alpha_j(t,x)a_j\Big)a_l^T}_{A_t(x)}
$$
This implies that $\|\nabla v_t(x)\|\leq \frac{1}{1-t} + \frac{\|A_t(x)\|}{(1-t)^3}$.
Thus, we need to bound $\|A_t(x)\|$. 
We find that (for the convenience of notation, we omit the argument $x$)
\begin{equation*}
A_t=\alpha_k(t)\Big((1-\alpha_k(t))a_k-\sum_{j\neq k}\alpha_j(t)a_j\Big)a_k^T+\sum_{l\neq k} \alpha_l(t)\Big(a_l-\sum_{j=1}^n\alpha_j(t)a_j\Big)a_l^T.
\end{equation*}
Using the notation $M=\max_{l=1,...,n} \|a_k\|$ and that $\alpha_j(t)$ sum to $1$, we obtain
\begin{align*}
\|A_t\|&\leq \alpha_k(t)M\Big\|(1-\alpha_k(t))a_k+\sum_{j\neq k}\alpha_j(t)a_j\Big\|+2M^2\sum_{l\neq k} \alpha_l(t)\\
&\leq 4M^2(1-\alpha_k(t))
\end{align*}
such that (writing the argument $x$ again) we have
$$
\|v_t(x)\|\leq \frac{1}{1-t} + \frac{4M^2(1-\alpha_k(t,x))}{(1-t)^3}
$$
Noting that $1-\alpha_k(t,x)\leq 1$ shows the global bound.
For the local bound, we find by \Cref{lem:alpha_abschaetzung} that there exist some $t_0\in(0,1)$ and $\delta>0$ such that for all $t\geq t_0$ and $x\in B_r(a_k)$ we have 
$$
1-\alpha_k(t,x) < \frac{n-1}{\exp(\frac{\delta}{2(1-t)^2})+n-1}.
$$
Since the exponential function is growing faster than any polynomial, we find that for any $\epsilon>0$ there exists some $1>t_1>t_0$ such that all $t>t_1$ it holds
$
1-\alpha_k(t,x)<\epsilon\frac{(1-t)^2}{4M^2}
$. In particular, for such $t$ and $x$, we obtain
$
\|\nabla v_t(x)\|\leq \frac{1+\epsilon}{1-t}
$. Choosing $\underline{t}=t_1$ yields the local bound.
\end{proof}

Next, we bound the gradient norm of $\gamma_t^{-1}$.
\begin{lem}\label{lem:inverse_Schranke}
There exists constants $C_1^*,C_2^*>0$ independent of $t$ (but depending on $a_1,...,a_N$) such that $\|\nabla[\gamma_t^{-1}](x)\|\leq \frac{C^*}{(1-t)^2}$ for all $t\in(0,1)$ and $x\in B_{C_2^*(1-t)}(a_k)$.
\end{lem}
\begin{proof}

For $w_s=v_{t-s}$ we consider the ODE $\dot \eta_s(x)=-w_s(\eta_s(x))$ with initial condition $\eta_0(x)=x$. We note that $\eta_t=\gamma_t^{-1}$.

\textbf{Step 1:}
First, we bound the velocity $w_s(x)$ locally around $a_k$ for small $s$. To this end, we note that 
$
v_t(x)=\frac1{1-t}\sum_{l} \alpha_l(t,x)(a_l-x)=\frac{1}{1-t}\Big(\sum_l\alpha(t-x)a_l-a_k\Big)+\frac{a_k-x}{1-t}.
$
In particular, we obtain with $M=\max_{l}\|a_k-a_l\|$ that
$$
\|v_t(x)\|\leq \frac{M}{1-t}\sum_{l\neq k}\alpha_l(t,x)+\frac{\|a_k-x\|}{1-t}=\frac{M(1-\alpha_k(t,x))}{1-t}+\frac{\|a_k-x\|}{1-t},
$$
or equivalently
$$
\|w_s(x)\|\leq \frac{M(1-\alpha_k(t-s,x))}{1-t+s}+\frac{\|a_k-x\|}{1-t+s}.
$$
Now let $r\leq \frac12\min_{l\neq k}\|a_k-a_l\|$. Then, we know by \Cref{lem:alpha_abschaetzung} that there exist some $t_0\in(0,1)$ and $\delta>0$ such that for all $t\geq t_0$ and $x\in B_r(a_k)$ we have $1-\alpha_k(t,x) < \frac{n-1}{\exp(\frac{\delta}{2(1-t)^2})+n-1}$. In particular, this implies that there exists some $1>t_1\geq t_0$ and $C>0$ such that for all $t\geq t_1$ it holds $1-\alpha_k(t,x) < C\exp(-\frac{\delta}{2(1-t)^2})$. Inserting this in the estimation above, we obtain for $s \leq t-t_1$ that
$$
\|w_s(x)\|\leq \frac{CM\exp(-\frac{\delta}{2(1-t+s)^2})}{1-t+s}+\frac{\|a_k-x\|}{1-t+s}.
$$

\textbf{Step 2:}
Next, we show that $\eta_s(x)$ remains close to $a_k$ for small $s$ whenever $x$ is close to $a_k$.
To this end let $r_t=\frac{r(1-t)}{2}$ and $x\in B_{r_t}(x)$.
Then, we define the auxiliary function $F_x(s)=\|\eta_s(x)-a_k\|$. To show that $F(s)<r$ for small $s$, we define $s_0=\min\{s\geq 0:F(s)\geq r\}$ (note that the minimum exists, since $F$ is continuous and fulfills $F(0)\leq r_t <r$). Then, it holds for all $s\leq \min(s_0,t-t_1)$ that
$$
\dot F_x(s)\leq \|w_s(\eta_s(x))\|\leq \frac{CM\exp(-\frac{\delta}{2(1-t+s)^2})}{1-t+s}+\frac{F_x(s)}{1-t+s}.
$$
In particular, we get by the Gronwall lemma that for $s\leq \min(s_0,t-t_1)$ it holds
\begin{align*}
F_x(s)&\leq CM(1-t+s)\int_0^s\frac{\exp(-\frac{\delta}{2(1-t+a)^2})}{(1-t+a)^2}d a
+\frac{1-t+s}{1-t}F(0)\\
&\leq \frac{2CMs}{\delta}+\frac{r_t}{1-t}=\frac{2CMs}{\delta}+\frac{r}{2}.
\end{align*}
Inserting $s^*=\min(s_0,t-t_1)$ this implies that
$
r=F(s^*)\leq \frac{2CMs^*}{\delta} +\frac{r}{2}
$
and therefore $s_0\geq s^*\geq \frac{\delta r}{4CM}$.
By the definition of $s_0$, we have shown that for $s<\min(\frac{\delta r}{4CM}, t-t_1)$ it holds that $F(s)<r$ and therefore $\eta_s(a_k)\in B_r(a_k)$.

\textbf{Step 3:} Finally, we show the actual result on the Jacobian of $\eta_t=\gamma_t^{-1}$.
For $J_s(x)=\nabla \eta_s(x)$  we know that it fulfills the ODE
$\dot J_s(x)=\nabla w_s(\eta_s(x))J_s(x)$ with initial condition $J_0(x)=I$.
For the norm, we obtain
$$
\frac{d}{d t}\|J_s(x)\|\leq \|\nabla w_s(\eta_s(x))\|\|J_s(x)\|
$$
such that the Gronwall lemma implies
$$
\|\nabla[\gamma_t^{-1}](x)\|=\|J_t(x)\|\leq \exp\Big(\int_0^t \|\nabla w_s(\eta_s(x))\| d s\Big)=\exp\Big(\int_0^t \|\nabla v_{t-s}(\eta_s(x))\| d s\Big).
$$
To bound this, let $\underline{t}$ be from \Cref{lem:gradient_von_v} for $\epsilon=1$ and the $r$ from above and set $t_2\coloneqq\max(\underline{t},t_1,1-\frac{\delta r}{4CM})$. If $t\leq t_2$, then we get by the global bound of \Cref{lem:gradient_von_v} for all $x\in\R^d$ that
\begin{align*}
\|\nabla[\gamma_t^{-1}](x)\|&\leq \exp\Big(\int_0^{t} \frac{c}{(1-t+s)^3} ds\Big)=\exp\Big(\int_0^{t} \frac{c}{(1-s)^3} ds\Big)\\
&=\exp\Big(\frac{c}{2(1-t)^2}-\frac{c}{2}\Big)\leq \exp\Big(\frac{c}{2(1-t_2)^2}-\frac{c}{2}\Big).
\end{align*}
For $t\geq t_2$, we obtain by Step 2 for $x\in B_{r_t}(x)$ that $\eta_s(x)\in B_r(a_k)$ for $s\leq t-t_2$ such that we can additionally apply the local bound of \Cref{lem:gradient_von_v}. This yields
\begin{align*}
\|\nabla[\gamma_t^{-1}](x)\|&\leq \exp\Big(\int_0^{t-t_2} \|\nabla v_{1-t-s}(\eta_s(x))\| ds+\int_{t-t_2}^t \frac{c}{(1-t+s)^3} ds\Big)\\
&\leq \exp\Big(\int_0^{t-t_2} \frac{2}{1-t+s} ds+\int_{0}^{t_2} \frac{c}{(1-s)^3} ds\Big)\\
&=\exp\Big(\int_{t_2}^{t} \frac{2}{1-s} ds+\frac{c}{2(1-t_2)^2}-\frac{c}{2}\Big)\\
&=\exp\Big(2\log(1-t_2)-2\log(1-t)+\frac{c}{2(1-t_2)^2}-\frac{c}{2}\Big)\\
&\leq\frac{\exp\Big(\frac{c}{2(1-t_2)^2}-\frac{c}{2}\Big)}{(1-t)^2}
\end{align*}
Combining both estimates and setting $C_1^*\coloneqq \exp\Big(\frac{c}{(1-t_2)^2}-\frac{c}{2}\Big)$ and $C_2^*=\frac{r}{2}$, we obtain that $\|\nabla[\gamma_t^{-1}](x)\|\leq \frac{C_1^*}{(1-t)^2}$ for all $x\in B_{C_2^*(1-t)}(a_k)$.
\end{proof}

Using this lemma, we can finally show the intended result.

\begin{prop}
     The curve $t\mapsto\gamma_t^{-1}(a_k)$ has finite length, i.e., $\int_0^1\|\frac{d}{dt}\gamma_t^{-1}(a_k)\|< \infty$. In particular, the limit $\lim_{t\to 1}\gamma_t^{-1}(a_k)$ exists.
\end{prop}
\begin{proof}
By the chain rule and inverse function theorem we obtain
$$
\frac{d}{dt} \gamma_t^{-1}(a_k)= -(\nabla \gamma_t(\gamma_t^{-1}(a_k)))^{-1} \dot \gamma_t(\gamma_t^{-1}(a_k))=-(\nabla \gamma_t(\gamma_t^{-1}(a_k)))^{-1} v_t(a_k)=-\nabla[\gamma_t^{-1}](a_k)v_t(a_k).
$$
Taking the norm and applying \Cref{lem:inverse_Schranke} gives
\begin{equation}\label{eq:Normzerlegung}
\|\frac{d}{dt} \gamma_t^{-1}(a_k)\|\leq \|\nabla[\gamma_t^{-1}](a_k)\|\|v_t(a_k)\|\leq \frac{C_1^*}{(1-t)^2}\|v_t(a_k)\|.
\end{equation}
To bound $\|v_t(a_k)\|$ we note that
$
v_t(a_k)=\frac{1}{1-t}\sum_{l\neq k}\alpha_l(t,a_k)(a_l-a_k)
$
which implies that
$
\|v_t(a_k)\|\leq (1-\alpha_k(t,a_k))\frac{M}{1-t}
$
for $M=\max_{l\neq k}\|a_k-a_l\|$.
Now, we know by \Cref{lem:alpha_abschaetzung} that there exist some $t_0\in(0,1)$ and $\delta>0$ such that for all $t\geq t_0$ it holds 
$
1-\alpha_k(t,a_k) < \frac{n-1}{\exp(\frac{\delta}{2(1-t)^2})+n-1}
$.
In particular, this implies that there exists some $1>t_1\geq t_0$ and $C>0$ such that for all $t\geq t_1$ it holds $1-\alpha_k(t,a_k) < (1-t)C\exp(-\frac{\delta}{2(1-t)^2})$ and therefore 
$
\|v_t(a_k)\|\leq CM\exp(-\frac{\delta}{2(1-t)^2})
$. For $t\leq t_1$ we still have the bound $v_t(a_k)\leq (1-\alpha_k(t,a_k))\frac{M}{1-t}\leq \frac{M}{1-t}$.
Inserting this in \eqref{eq:Normzerlegung}, we get that
$$
\|\frac{d}{dt} \gamma_t^{-1}(a_k)\|\leq \begin{cases}
\frac{C_1^*M}{(1-t)^3}&\text{for }t\leq t_1,\\
\frac{C_1^*CM\exp(-\frac{\delta}{2(1-t)^2})}{(1-t)^2}.
\end{cases}
$$
Integrating over $t$ gives
$$
\int_0^1 \|\frac{d}{dt} \gamma_t^{-1}(a_k)\| dt=\int_0^{t_1}\frac{C_1^*M}{(1-t)^3} dt+\int_{t_1}^1\frac{C_1^*CM\exp(-\frac{\delta}{2(1-t)^2})}{(1-t)^2} dt,
$$
which is finite, since the first integrand is bounded from above by $\frac{C_1^*M}{(1-t_1)^3}$ and the second integrand is bounded from above by $\frac{2C_1^*CM}{\delta}$.
\end{proof}

We note that for $t\in(0,1)$ large enough it is directly clear from \Cref{prop:into_the_ball_implies_convergence_to_ak} that $\gamma_t^{-1}(a_k)\in\Gamma_k$. Due to the finite-length  of $t\mapsto \gamma_t^{-1}(a_k)$ this implies in the limit $\lim_{t\to 1}\gamma_t^{-1}(a_k)\in\overline{\Gamma_k}$.

\begin{prop}\label{prop:contractible2}
Assume that $\lim_{t\to 1}\gamma_t^{-1}(a_k)\in\Gamma_k$ (and not only in $\overline{\Gamma_k}$).
Then, the set $\Gamma_k$ is contractible. 
\end{prop}
\begin{proof}
We consider for $r<\frac12\min_{l\neq k}\|a_k-a_l\|$ and the corresponding $\underline{t}$ from \Cref{prop:into_the_ball_implies_convergence_to_ak} the mapping
$$
H(t,x)=\begin{cases}
\gamma_t^{-1}(\phi(t,\gamma_t(x))\gamma_t(x)+(1-\phi(t,\gamma_t(x)))a_k),&\text{for }t<1,\\
\lim_{t\to 1}\gamma_t^{-1}(a_k),&\text{for }t=1,
\end{cases}
$$
where
$$
\phi(t,x)=\begin{cases}1,&\text{if }t<\underline{t}\text{ or }x\not\in B_r(a_k),\\
(1-\psi(\|x-a_k\|))+\psi(\|x-a_k\|)(\frac{1-t}{1-\underline{t}})^3,&\text{if }t\geq\underline{t}\text{ and }x\in B_{r}(a_k)\setminus B_{r/2}(a_k),\\
(\frac{1-t}{1-\underline{t}})^3,&\text{if }t\geq\underline{t}\text{ and }x\in B_{r/2}(a_k),\end{cases}
$$
for a continuous decreasing function $\psi\colon [r/2,r]\to[0,1]$ with $\psi(r/2)=1$ and $\psi(r)=0$. We show that $H$ is a homotopy between the identity on $\Gamma_k$ and a constant map. It is directly clear from the definition that $H(0,x)=x$ and that $H(1,x)=\lim_{t\to 1}\gamma_t^{-1}(a_k)$ is constant. Next, we show that $H(t,x)\in\Gamma_k$ for all $(t,x)\in[0,1]\times\Gamma_k$: For $t=1$, we know by assumption that $H(t,x)=\lim_{t\to1}\gamma_t^{-1}(a_k)\in\Gamma_k$. If $t<\underline{t}$ or $\|\gamma_t(x)-a_k\|\geq r$, we get that $H(t,x)=x\in\Gamma_k$. If $t\geq \underline{t}$ and $\gamma_t(x)\in B_r(a_k)$, we obtain by the definition of $H$ and \Cref{prop:into_the_ball_implies_convergence_to_ak} that $H(t,x)\in\gamma_t^{-1}(B_r(a_k))\subseteq\Gamma_k$.
Thus, it remains to show that $H$ is continuous on $[0,1]\times \Gamma_k$. 
By definition, $H$ is on $[0,1)\times\Gamma_k$ the concatenation of continuous maps and therefore continuous. Thus, it remains to show that for $\hat x\in\Gamma_k$ and a sequence $(t_i,x_i)_i$ with $t_i\to 1$ and $x_i\to \hat x$ it holds $H(t_i,x_i)\to \lim_{t\to 1}\gamma_t^{-1}(a_k)$. 

We first apply \Cref{lem:compact_sets_in_Br} to the compact set $\{x_i:i=1,2,...\}\cup\{\hat x\}\subseteq\Gamma_k$. This yields that there exists some $t_0^*<1$ such that $\gamma_t(x_i)\in B_{r/2}(a_k)$ for all $i$ and $t\geq t_0^*$.

Next, we use \Cref{lem:inverse_Schranke} to show the convergence of $H(t_i,x_i)$. Since $t_i\to 1$, we find $i_0$ such that for all $i\geq i_0$ we have $t_i\geq t_0^*$. Then, we obtain that $\gamma_{t_i}(x_i)\in\overline{B_{r/2}(a_k)}$ such that $H(t_i,x_i)=\gamma_{t_i}^{-1}(a_k+(\frac{1-t_i}{1-\underline{t}})^3(\gamma_{t_i}(x_i)-a_k))$. Again since $t_i\to 1$ we can choose $i_1\geq i_0$ such that for all $i\geq i_1$ it holds that $(\frac{1-t_i}{1-\underline{t}})^3< \frac{2C_2^*}{r}(1-t_i)$ for $C_2^*$ from \Cref{lem:inverse_Schranke}. In particular, this yields that for $i\geq i_1$ it holds that $\|(\frac{1-t_i}{1-\underline{t}})^3(\gamma_{t_i}(x_i)-a_k)\|< C_2^*(1-t_i)$. Since by \Cref{lem:inverse_Schranke}, $\gamma_{t_i}$ is $\frac{C_1^*}{(1-t_i)^2}$-Lipschitz continuous on $B_{C_2^*(1-t)}(a_k)$, this implies that
$$
\|H(t_i,x_i)-\gamma_{t_i}^{-1}(a_k)\|\leq \frac{C_1^*}{(1-t_i)^2}(\frac{1-t_i}{1-\underline{t}})^3\|\gamma_{t_i}(x_i)-a_k\|\leq \frac{C_1^*r(1-t_i)}{2(1-\underline{t})^3}\to 0
$$
as $i\to\infty$. In particular, we have that
$$
\|H(t_i,x_i)-\lim_{t\to1}\gamma_t^{-1}(a_k)\|\leq \|H(t_i,x_i)-\gamma_{t_i}^{-1}(a_k)\|+\|\gamma_{t_i}^{-1}(a_k)-\lim_{t\to1}\gamma_t^{-1}(a_k)\|\to 0
$$
which shows that $H$ is continuous and therefore a homotopy connecting the identity on $\Gamma_k$ and a constant map.
\end{proof}

We would like to show that $\Gamma_k$ is in fact homeomorphic to $B_1(0)$ (or equivalently to $\R^d$). In $d=1$ and $d=2$ this is true for any open and contractible set such that this directly follows from \Cref{prop:contractible}. However, for $d\geq3$, we need an additional property given in the following lemma.

\begin{lem}\label{lem:simply_connected_at_infitinity}
Let $d\geq 3$. Then, $\Gamma_k$ is simply connected at infinity, i.e., for any compact set $K\subseteq\Gamma_k$ there exists another compact set $L\subset\Gamma_k$ with $K\subseteq L^\circ$ such that any loop in $\Gamma_k\setminus L$ is contractible in $\Gamma_k\setminus K$.
\end{lem}
\begin{proof}
Fix $r<\frac12\min_{l\neq k}\|a_k-a_l\|$.
By \Cref{lem:compact_sets_in_Br}, we know that there exists $t_0$ such that $\gamma_t(K)\subseteq B_{r/2}(a_k)$ for all $t\geq t_0$. Moreover, we know by \Cref{prop:into_the_ball_implies_convergence_to_ak}, that there exist $\underline{t}$ such that $B_r(a_k)\subset \gamma_t(\Gamma_k)$ for all $t\geq \underline{t}$. Now let $t^*=\max(t_0,\underline{t})$ and define $L=\gamma_{t^*}^{-1}(\overline{B_{r/2}(a_k)})$. Note that since $\gamma_{t^*}$ is a bi-Lipschitz diffeomorphism, we have that $L$ is compact and fulfills $L^\circ=\gamma_{t^*}^{-1}(B_{r/2}(a_k))\supseteq K$.

Now let $f\colon S^1\to \Gamma_k\setminus L$ be a loop, and define the map $H(x,t)\colon S^1 \times [0,1]\to \Gamma_k\setminus K$ for $t\leq \frac12$ by 
\begin{align*}
H(x,t)&=\gamma_{t^*}^{-1}\left(\left((1-2t)+\frac{rt}{\|\gamma_{t^*}(f(x))\|}\right)\gamma_{t^*}(f(x))+2t a_k\right)\\
&\in \gamma_{t^*}^{-1}\left(B_r(a_k)\setminus B_{r/2}(a_k)\right)\subseteq \Gamma_k\setminus K.
\end{align*}
Note that by definition, for $t=\frac12$, we have that $\gamma_{t^*}(H(x,t))=\frac{r\gamma_{t^*}(f(x))}{2\|\gamma_{t^*}(f(x))\|}\in \partial B_{r/2}(a_k)$. Further, we know since $d\geq 3$ that $\partial B_{r/2}(a_k)$ is simply connected. Thus, we know that there exists a homotopy $\tilde H\colon S^1\times[0,1]\to \partial B_{r/2}$ such that $\tilde H(x,0)=\frac{r\gamma_{t^*}(f(x))}{2\|\gamma_{t^*}(f(x))\|}$ and such that $\tilde H(x,1)$ is constant.
Thus, we can define for $t\geq \frac12$
$$
H(x,t)=\gamma_{t^*}^{-1}(\tilde H(x,2t-1))\in \gamma_{t^*}^{-1}(\partial B_{r/2}(a_k))\subseteq \Gamma_k\setminus K.
$$
Summarizing it holds by definition that $H(x,t)\in\Gamma_k\setminus K$, that $H(x,1)$ is constant and that $H$ is continuous as the concatenation of continuous maps. Consequently $H$ is a homotopy and since $K$ was chosen arbitrarily, we obtain that $\Gamma_k$ is simply connected at infinity.
\end{proof}

Now we can conclude that $\Gamma_k$ is topologically equivalent to a ball.

\begin{cor}
Assume that $\lim_{t\to 1}\gamma_t^{-1}(a_k)\in\Gamma_k$. Then, $\Gamma_k$ is homeomorphic to $\R^d$ and to $B_1(0)$.
\end{cor}
\begin{proof}
For $d=1$ and $d=2$ any open and contractible subset of $\R^d$ is homeomorphic to $\R^d$ (and thus also to $B_1(0)$) such that the claim follows from \Cref{prop:contractible2}. For $d\geq3$, we additionally know from \Cref{prop:contractible2} and \Cref{lem:simply_connected_at_infitinity} that $\Gamma_k$ is simply connected at infinity such that the claim follows from the fact that any set which is open, contractible and simply connected at infinity is homeomorphic to $\R^d$ (and to $B_1(0)$), see \cite{BT1989} for $d=3$, \cite{FQ1990} for $d=4$ and \cite{Stallings1962} for $d>4$.
\end{proof}

\section{Proofs for Geometry of the FM Cells}
\label{appendix:proofs_geometry}

\subsection{Proofs for the Four-Point Counterexample}
\label{app:counterex_proofs}

\begin{proof}[Proof of \Cref{lem1:regular_triangle_example}]
By symmetry it suffices to show the claim for $k=3$ for which the line segment reads as $L_3 \eqdef \{(-c,0):c\geq0\}$ (See red line in Figure \ref{fig:counterex} (c)). 

First we show that for any initialization $x \in L_3$, its image by the flow map $\gamma_t(x) \in L_3$.  First, note that by definition it holds for $x=(b,0)$ with $b\in\R$ that $\alpha_1(t,x)=\alpha_2(t,x)$. In particular, we get that $v_t(b,0)=(a,0)$ for some $a\in\R$. Moreover, we get for any $t$ that $\alpha_1(t,(0,0))=\alpha_2(t,(0,0))=\alpha_3(t,(0,0))$ such that $v_t(0,0)=\alpha_1(t,(0,0))(x_1+x_2+x_3)=0$. From this, we can conclude for $b\geq0$ that $\gamma_t(-b,0) \in L_3$. 

Now we prove that for $x\in L_3$, $\gamma_t(x)$ actually converges to $a_4 = (0,0)$ when $t \to 1$. Let $b_0\leq 0$ and let $b_t$ such that $(b_t,0)=\gamma_t(b_0,0)$. We show that $b_t$ converges to zero for $t\to 1$ in three steps.

\textbf{Step 1:} We show that for all $\epsilon>0$ there exists some time $t_0$ such that $b_{t}\geq-\frac12-\epsilon$ for all $t\geq t_0$. To this end, we observe
\begin{align*}
(1-t)\dot b_t&=(\alpha_1(t,(b_t,0))+\alpha_2 (t,(b_t,0)))(-\frac12-b_t)\\
&\quad+\alpha_3(t,(b_t,0))(1-b_t) -\alpha_4(t,(b_t,0))b_t\\
&\geq \sum_{k=1}^4\alpha_k(t,(b_t,0))(-\frac12-b_t)=-\frac12-b_t.
\end{align*}
Applying the Gronwall lemma on $\dot b_t\geq - \frac{\frac12+b_t}{1-t}$ gives $b_t \geq (1-t)b_0-\frac{t}2 \xrightarrow[t \to 1]{} -\frac12$, which shows step 1. 

\textbf{Step 2:} Next, we show that for $\epsilon = 0.1$ there exists some $t_1$ 
such that $v_t(b,0)\geq \frac{c}{1-t}$ for some $c>0$ and all $b\in -\frac12+[-\epsilon,\epsilon]$.
Using the formula for $v_t$ from the previous step, we get that $v_t(b,0)=(w_t(b),0)$ with
\begin{align*}
(1-t)w_t(b)&\geq(1-\alpha_4(t,(b,0)))(-\frac12-b) -\alpha_4(t,(b,0))b\\
&=-\frac{1-\alpha_4(t,(b,0))}{2}-b
\geq \frac{1}{2}-\epsilon -\frac{1-\alpha_4(t,(b,0))}{2}
\end{align*}
Observing that $\alpha_4(t,(b,0))$ converges to $1$ uniformly on $\{(c,0):c\in-\frac12+[-\epsilon,\epsilon]\}$ yields the claim of step 2.

\textbf{Step 3:} Set $t_2=\max(t_0,t_1,\underline{t}_r)$ with $\underline{t}_r$ from \Cref{prop:into_the_ball_implies_convergence_to_ak} and $r=\frac12-\epsilon$. We show that there exists some $t\geq t_2$ such that $b_t\geq -\frac12+\epsilon$.
We get from step 1 that $b_{t_2}\geq -\frac12-\epsilon$.  
To this end, we build a contradiction. Assume that $b_t<-\frac12+\epsilon$ for all $t\in [t_2,1)$. Then, we get by step 1 that $b_t\in-\frac12+[-\epsilon,\epsilon]$. By step 2 this implies that $\dot b_t\geq \frac{c}{1-t}$ for some $c > 0$. Applying the Gronwall lemma on this leads for $t\in[t_2,1)$ to $b_t\geq b_{t_2}+c \log(\frac{1-t_2}{1-t})$. This converges to $\infty$ for $t\to 1$ whcih contradicts the assumption that $b_t<-\frac12+\epsilon$.

\textbf{Conclusion:}
By step 3, we know that there exists some $t\geq \underline{t_r}$ such that $(b_t,0)\in B_r(0,0)$. Applying \Cref{prop:into_the_ball_implies_convergence_to_ak} yields $b_1=0$ such that $(b_0,0)\in\Gamma_4$.
\end{proof}

\begin{proof}[Proof of \Cref{lem2:regular_triangle_example}]
By symmetry it suffices to show the claim for $k=3$ for which $a_k=(1,0)$. For $b_0\in\R$, we consider the curve $\gamma_t(b_0,0)=(b_t,0)$. We note that for any $t$ we can estimate
\begin{align*}
(1-t)\dot b_t&=(\alpha_1(t,(b_t,0))+\alpha_2 (t,(b_t,0)))(-\frac12-b_t)\\
&\quad+\alpha_3(t,(b_t,0))(1-b_t) -\alpha_4(t,(b_t,0))b_t\\
&\begin{cases}
    \geq -\frac12-b_t\\
    \leq 1-b_t
\end{cases}
\end{align*}
With the Gronwall lemma for both estimates this implies that
\begin{equation}\label{eq:regular_triangle_example_estimate}
(1-t)b_0 -\frac{t}2\leq b_t\leq (1-t)b_0 + t.  
\end{equation}
Now we note that by Corollary~\ref{cor:stable_balls} and \Cref{prop:into_the_ball_implies_convergence_to_ak}, we deduce that for $r=1/4$ there exists a time $t_0$ such that, for all $t\geq t_0$, the implication
\[
1-r \leq b_t \leq 1+r \quad \Rightarrow \quad 1-r \leq b_s \leq 1+r \text{ for all } s \geq t
\]
holds, and moreover $b_1 = 1$. 
Now choose $c_0=\frac{1-r}{1-t_0}+\frac{t_0}{2(1-t_0)}$ such that $(1-t_0)c_0-\frac{t_0}{2} = 1-r$. Then it holds by \eqref{eq:regular_triangle_example_estimate} for $b_0\geq c_0$ that $b_{t_0}\geq (1-t)b_0-\frac{t_0}{2}\geq 1-r$. In particular, we have that $b_t\geq 1-r$ for all $t\geq t_0$. At the same time the other estimate in \eqref{eq:regular_triangle_example_estimate} gives that there exists $t\geq t_0$ such that $b_t\leq 1+r$. Together, we obtain that there exists $t\geq t_0$ such that $1-r\leq b_t \leq 1+r$ such that $b_1=1$.
\end{proof}

\subsection{Proof of the Non-Monotonicity}
\label{app:non_monotony}

\begin{proof}[Proof of \Cref{lem3:regular_triangle_example}]
   Consider the two points $x_1 = 
    c_1 a_3 = (c_1,0)$ with $c_1 > c_0$ where $c_0$ is given by \Cref{lem2:regular_triangle_example} and $x_2 = (-c_2 \cos(-\frac{4\pi}{3}), -c_2 \sin(\frac{4 \pi}{3})) = \frac{c_2}{2}(1, \sqrt{3})$ such that $\frac{c_2}{2}-c_1 > 0$. 
    \begin{itemize}
        \item By \Cref{lem2:regular_triangle_example},  $\gamma_t(x_1) \xrightarrow[t \to 1]{} a_3$ and by Corollary~\ref{cor:stable_balls}, there exists $t_0 > 0$ such that for $t \geq t_0$, $\gamma_t(x_1) = (b_t,0)$ with $b_t \geq 1 - r$. 
        \item By \Cref{lem1:regular_triangle_example},  $\gamma_t(x_2) \xrightarrow[t \to 1]{} a_4 = (0,0)$ and by Corollary~\ref{cor:stable_balls}, there exists $t_1 > 0$ such that for $t \geq t_0$, $\gamma_t(x_2) = \frac{c_t}{2}(1, \sqrt{3})$ and $\norm{\gamma_t(x_2)} = c_t \leq r$
    \end{itemize}
    We then get for $t \geq t_m \eqdef \max(t_0,t_1)$, 
    \begin{align}
        \langle \gamma_t(x_2)  - \gamma_t(x_1), x_2 - x_1 \rangle &= \left(\frac{c_t}{2} - b_t\right)\left(\frac{c_2}{2} - c_1\right) + \frac{3 c_t c_2}{4} \\
        &\leq \left(\frac{r}{2} - (1-r)\right)\left(\frac{c_2}{2} - c_1\right)+ \frac{3 r c_2}{4} \\
        &= - \frac{1}{2}\Big((1-3r)c_2 - (2-3r)c_1 \Big) .
    \end{align}
    Then for any $c_2 = K c_1$ with $K > \frac{2-3r}{1-3r}$, we get
     \begin{align}
     \langle \gamma_t(x_2)  - \gamma_t(x_1), x_2 - x_1 \rangle <  - \frac{c_1}{2}\Big((1-3r)K - (2-3r) \Big).
     \end{align}
\end{proof}

\begin{proof}[Proof of Lemma 12]
Let $x_1,x_2\in \mathbb{R}^d$ and $c>0$ be given by Lemma 11, so that
$$
\langle \gamma(x_2)-\gamma(x_1),x_2-x_1\rangle<-c .
$$
By Proposition 3, for each fixed $x_i \in \{ x_1, x_2\}$, we have
$\gamma^\varepsilon(x_i)\to \gamma(x_i)$ as $\varepsilon\to 0$. Hence, for all
$\varepsilon>0$ small enough,
$$
\langle \gamma^\varepsilon(x_2)-\gamma^\varepsilon(x_1),x_2-x_1\rangle
<-\frac{c}{2}.
$$
Set $h:=x_2-x_1$. Since $\gamma^\varepsilon$ is smooth, the fundamental theorem
of calculus gives
$$
\langle \gamma^\varepsilon(x_2)-\gamma^\varepsilon(x_1),h\rangle
=
\int_0^1
\left\langle
J_{\gamma^\varepsilon}(x_1+sh)h,h
\right\rangle ds .
$$
Equivalently,
$$
\langle \gamma^\varepsilon(x_2)-\gamma^\varepsilon(x_1),h\rangle
=
\frac12\int_0^1
h^\top
\left(
J_{\gamma^\varepsilon}(x_1+sh)
+
J_{\gamma^\varepsilon}(x_1+sh)^\top
\right)
h\,ds <-\frac{c}{2}.
$$
Therefore there exists $s_0\in[0,1]$ such that
$$
h^\top
\left(
J\gamma^\varepsilon(x_1+s_0h)
+
J\gamma^\varepsilon(x_1+s_0h)^\top
\right)
h
<
-c .
$$
It follows that
$$
\lambda_{\min}\left(
J\gamma^\varepsilon(x_1+s_0h)
+
J\gamma^\varepsilon(x_1+s_0h)^\top
\right)
<
-\frac{c}{\|h\|^2}.
$$
Thus, setting $c':=c/\|x_2-x_1\|^2>0$, we obtain
$$
\inf_x
\lambda_{\min}\left(
J\gamma^\varepsilon_1(x)+J\gamma^\varepsilon_1(x)^\top
\right)
<-c'.
$$
\end{proof}

\section{Higher Dimensional Training Example (MNIST)}
\label{app:mnist}

We complement the two-dimensional examples from the main text with a training experiment on MNIST. 
We train a VAE encoder, parameterized by a two-layer MLP, to embed MNIST images from dimension $784$ into a latent space of dimension $16$. 
We then select $100$ latent training points $(a_k)_{k=1}^{100}$, with $10$ points from each digit class, and use them as the atoms of the semi-discrete target. 
On this dataset, we compare the semi-discrete OT Laguerre cells, the exact Flow Matching cells computed from the closed-form velocity, and the cells induced by a trained Flow Matching velocity field, parameterized by a $2$-layer MLP with $256$ hidden neurons.

To visualize these $16$-dimensional cells, we restrict the comparison to a two-dimensional affine slice spanned by the first two PCA directions of the selected latent points. 
For the trained Flow Matching model, terminal points do not necessarily coincide exactly with the atoms $a_k$, reflecting the usual generalization and approximation effects of the learned velocity. 
We therefore assign each grid point to the atom closest to its terminal image under the learned flow.

Figure 7 shows that the qualitative differences between OT and FM persist in this higher-dimensional latent setting. 
Compared with Laguerre cells, the FM cells have curved boundaries, are often non-convex, and display different locations and neighborhood relations. 
Moreover, the trained FM cells closely resemble the exact FM cells, suggesting that even a simple learned velocity can recover the main geometric features of the exact semi-discrete Flow Matching tessellation.
\begin{figure}[h]
    \centering
    \begin{subfigure}{\threeimagesfigure}
    \centering
    \includegraphics[width=\linewidth]{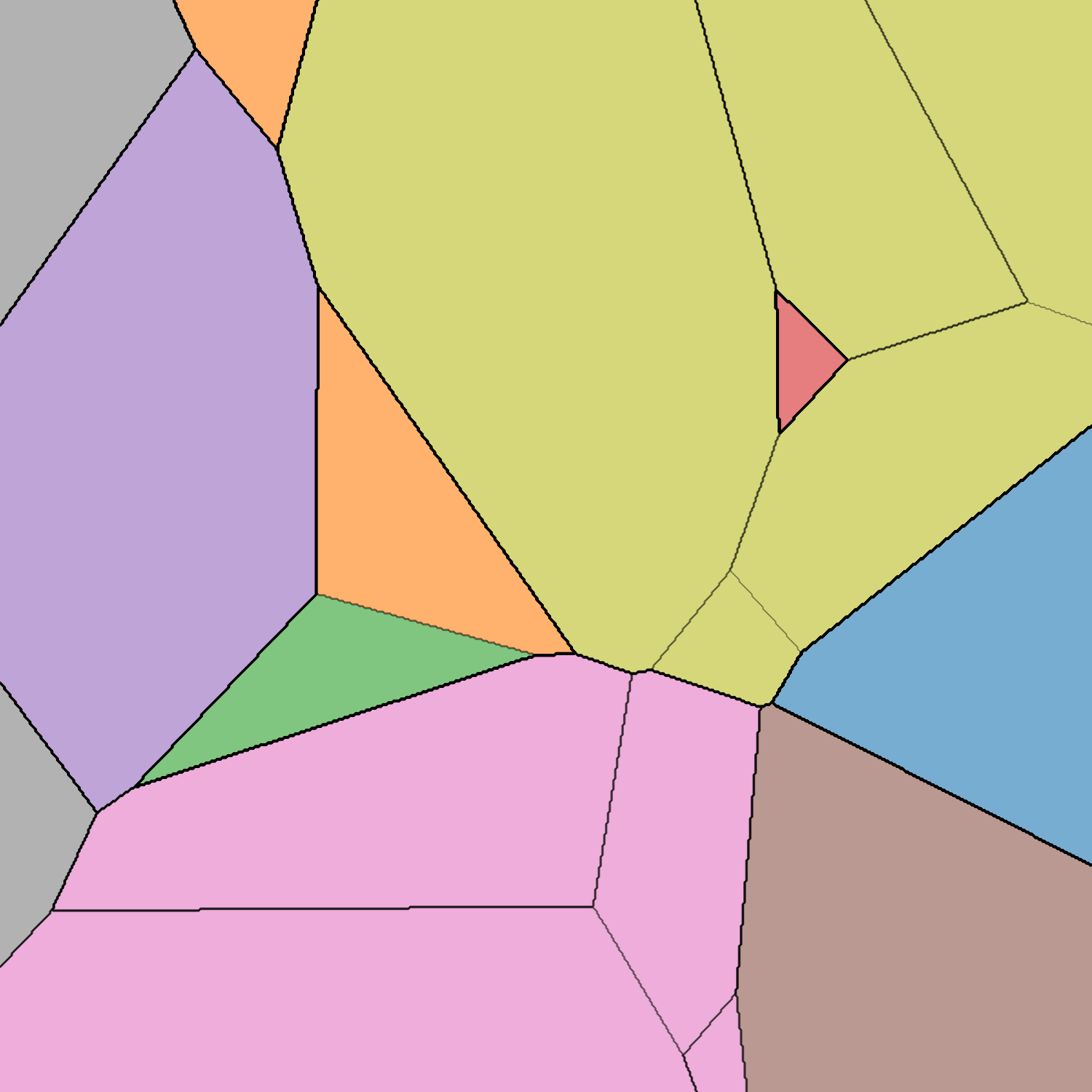}
     \caption{OT}
     \end{subfigure}
    \begin{subfigure}{\threeimagesfigure}
    \centering
    \includegraphics[width=\linewidth]{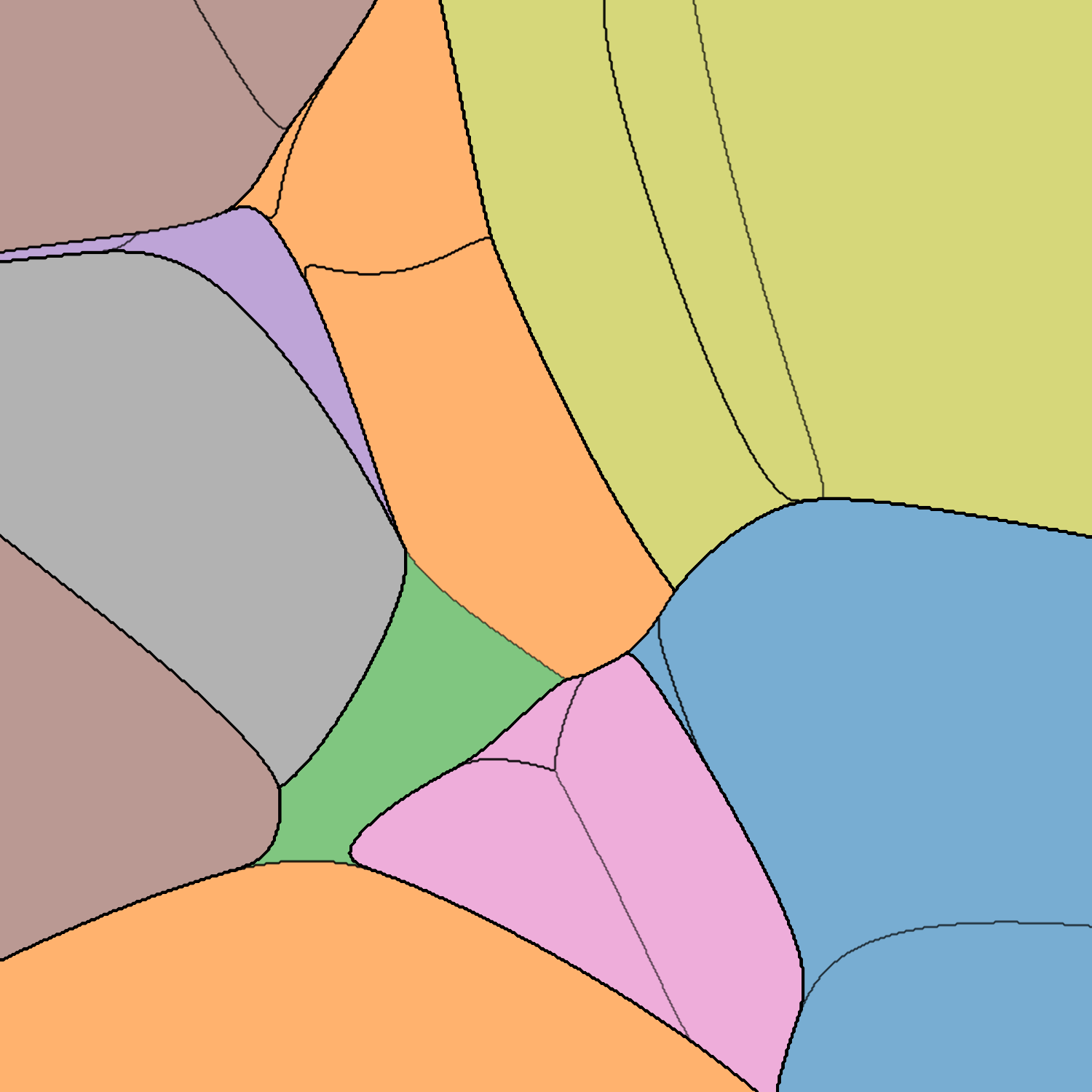}
     \caption{FM with exact velocity}
     \end{subfigure}
     \begin{subfigure}{\threeimagesfigure}
    \centering
    \includegraphics[width=\linewidth]{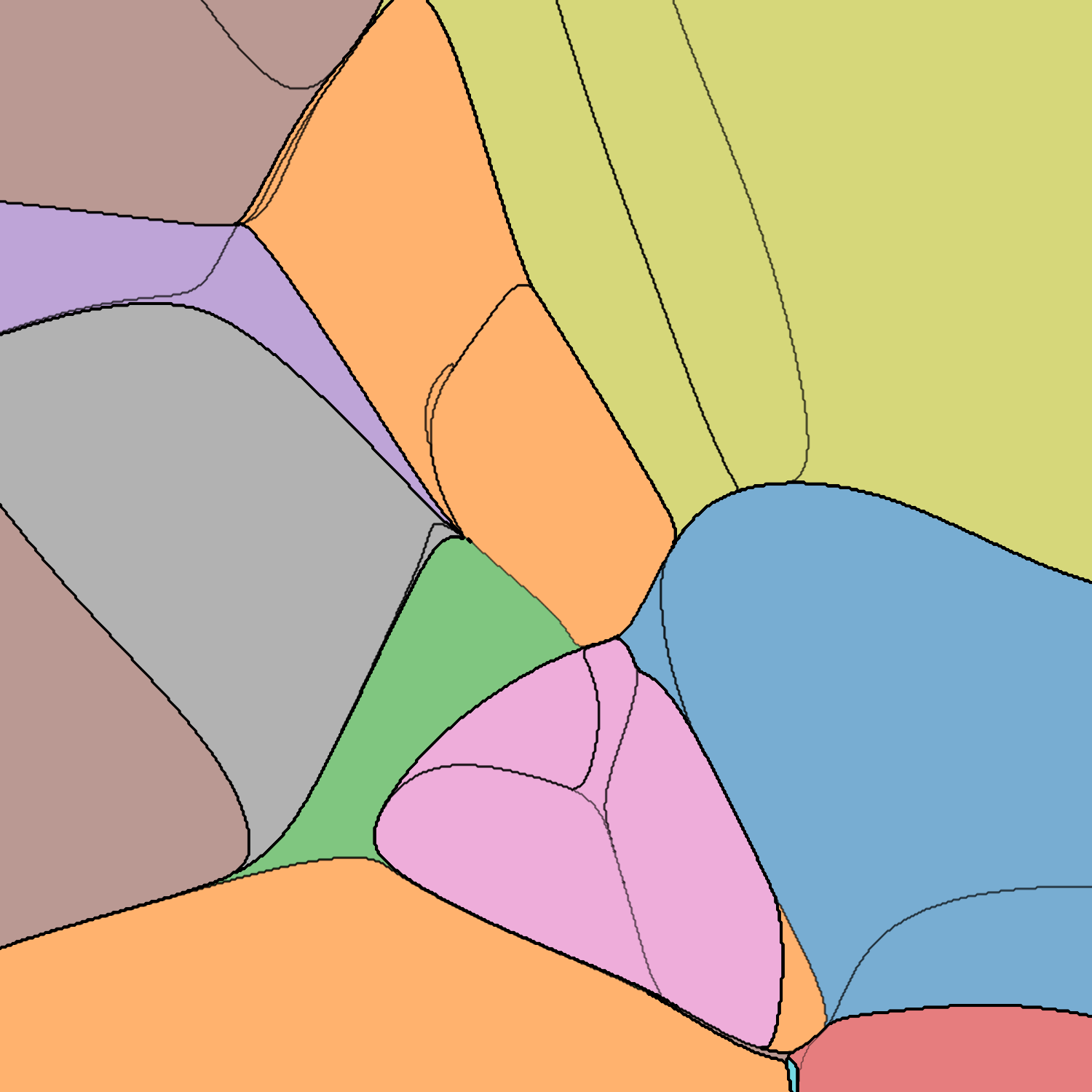}
     \caption{FM with learned velocity}
     \end{subfigure}

     \caption{ Semi-discrete assignment cells $\Gamma_k$ for $100$ MNIST training points in a VAE latent space of dimension $16$. Cells are visualized on a two-dimensional affine slice spanned by the first two PCA directions.  Cells are colored according to the MNIST digit label of the corresponding target point. 
    We compare (a) Optimal Transport Laguerre cells, (b) Flow Matching cells induced by the closed-form velocity, and (c) Flow Matching cells obtained from a neural-network approximation of the velocity. 
    In panel (c), each grid point is assigned the color of the target point closest to its image under the terminal flow.}
    \label{fig:MNIST}
\end{figure}

\end{appendices}

\end{document}